\documentclass{article}
\usepackage{lmodern}
\usepackage{iclr2027_conference,times}
\usepackage{amsmath,amssymb,amsthm,graphicx,booktabs}
\usepackage{algorithm,algorithmic}
\usepackage{xcolor,hyperref,url,wrapfig,colortbl}
\usepackage{microtype}
\usepackage{placeins}
\usepackage{mathtools}
\usepackage{nicematrix}
\definecolor{CoEMBlue}{HTML}{E8F1F8}

\newif\ifcoemprose
\coemprosetrue
\AddToHook{env/tabular/begin}{\coemprosefalse}
\AddToHook{env/algorithmic/begin}{\coemprosefalse}
\AddToHook{env/minipage/begin}{\coemprosefalse}
\AddToHook{env/thebibliography/begin}{\coemprosefalse}
\AddToHook{para/begin}{\ifcoemprose\parfillskip=0pt plus 0.30\linewidth\relax\fi}
\newtheorem{proposition}{Proposition}
\newtheorem{theorem}{Theorem}

\newtheorem{corollary}{Corollary}

\usepackage{booktabs}
\usepackage{multirow}
\usepackage[table]{xcolor}
\usepackage{xspace}

\newcommand{\our}{\textsc{CoEM}\xspace}

\hypersetup{colorlinks=true,linkcolor=blue!45!black,citecolor=blue!45!black,urlcolor=blue!45!black,pdfauthor={Jingguang Li, Yebo Wu, Zuyi Guo, Kailang Ma, Xianjie DAI, Han Zheng, Benwang Chen, Li Li, Can Rong, Heye Huang},pdftitle={CoEM: Empowering Long-Context Reasoning with Commit-on-Evidence Memory}}
\title{CoEM: Empowering Long-Context Reasoning with Commit-on-Evidence Memory}
\author{Jingguang Li\textsuperscript{1}, Yebo Wu\textsuperscript{2}, Zuyi Guo\textsuperscript{1}, Kailang Ma\textsuperscript{1}, Xianjie DAI\textsuperscript{1}\\
\bfseries Han Zheng\textsuperscript{3}, Benwang Chen\textsuperscript{1}, Li Li\textsuperscript{2}, Can Rong\textsuperscript{3}, Heye Huang\textsuperscript{1}\thanks{Corresponding author: \href{mailto:heye.huang@kaist.ac.kr}{\nolinkurl{heye.huang@kaist.ac.kr}}.}\\[3pt]
\normalfont\small\textsuperscript{1}Korea Advanced Institute of Science \& Technology\\
\normalfont\small\textsuperscript{2}University of Macau\\
\normalfont\small\textsuperscript{3}Massachusetts Institute of Technology}
\iclrfinalcopy
\begin{document}
\maketitle
\fancyhead{}
\renewcommand{\headrulewidth}{0pt}
\begin{abstract}
Long-context reasoning is essential for complex and long-horizon tasks, yet the performance of large language models (LLMs) degrades as context length increases. Recent approaches address this by processing input chunk by chunk while maintaining a bounded textual memory in model context. However, premature information compression can discard critical details essential for subsequent reasoning.
In this paper, we introduce Commit-on-Evidence Memory (\our), which learns when to convert source evidence into compact memory facts. Specifically, under a fixed context-memory budget, \our preserves potentially useful source excerpts verbatim in a pending set, allowing subsequent context to clarify their relevance before irreversible compression. As new context arrives, a learned policy revisits each pending excerpt and decides whether to promote it to the committed memory, retain it for further consideration, or discard it. 
A frozen verifier ensures proposed facts are accepted only if supported by retained excerpts and current context. To further guide effective memory management, we train this policy using reinforcement learning by combining fine-grained, step-level evidence rewards with final answer rewards.
Extensive experiments demonstrate that \our consistently improves long-context reasoning. When evaluated on 6,400 documents long-context input, \our outperforms the strongest memory baseline by 10.4–11.4 F1 points on Qwen3.5-9B. Code repository: \url{https://github.com/benmagnifico/CoEM}.

\end{abstract}
\section{Introduction}
\label{sec:introduction}

Long-context reasoning is fundamental to complex and long-horizon tasks such as cross-document multi-hop question answering, repository-level debugging, and scientific evidence synthesis, where the evidence needed to reach a conclusion is often distributed across distant passages~\citep{bai-etal-2024-longbench,bai-etal-2025-longbench,wu2026devft,wu2026tsembed}. 
However, simply extending the context window does not ensure effective reasoning. As the length of input grows, relevant evidence becomes increasingly sparse among distractors, so that LLMs struggle to identify and integrate dependencies over long distances, leading to substantial performance degradation~\citep{hsieh2024ruler,liu2025comprehensivesurveylongcontext,wu2026fedllmsurvey}. Processing the full text altogether also incurs rapidly increasing computational and memory costs of training and inference, and sufficiently long inputs may exceed the model's context window and lead to task failure~\citep{warner-etal-2025-smarter,wu2026transcending,wu2026smartfed}. 

\vspace{-2mm}
To address these challenges, recent work has explored a recurrent-memory paradigm that processes long contexts sequentially in an RNN-like manner~\citep{yu2026memagent}. At each step, the memory agent summarizes important information from the current chunk into a bounded textual memory in its context window and carries the memory to the next step. This protocol bounds retained information in the memory and avoids repeated source processing.
Subsequent methods have explored memory callbacks that revisit historical memory snapshots to mitigate information loss from overwriting, as well as developed gated recurrent memory that skip unnecessary memory updates and terminate reading once sufficient evidence has been gathered~\citep{shi2025rememr1,sheng2026grumem,wu2025fedpruner,wu2026chainfed}.
Despite these advances, as shown in the left and middle part of Figure~\ref{fig:limitation}, immediate memory-update decisions risk discarding seemingly irrelevant details that later become essential for reasoning.
Under this protocol, this details omitted from memory cannot be recovered once discarded, ultimately leading to reasoning failure.

\begin{figure}[t]
\centering
\includegraphics[width=\linewidth]{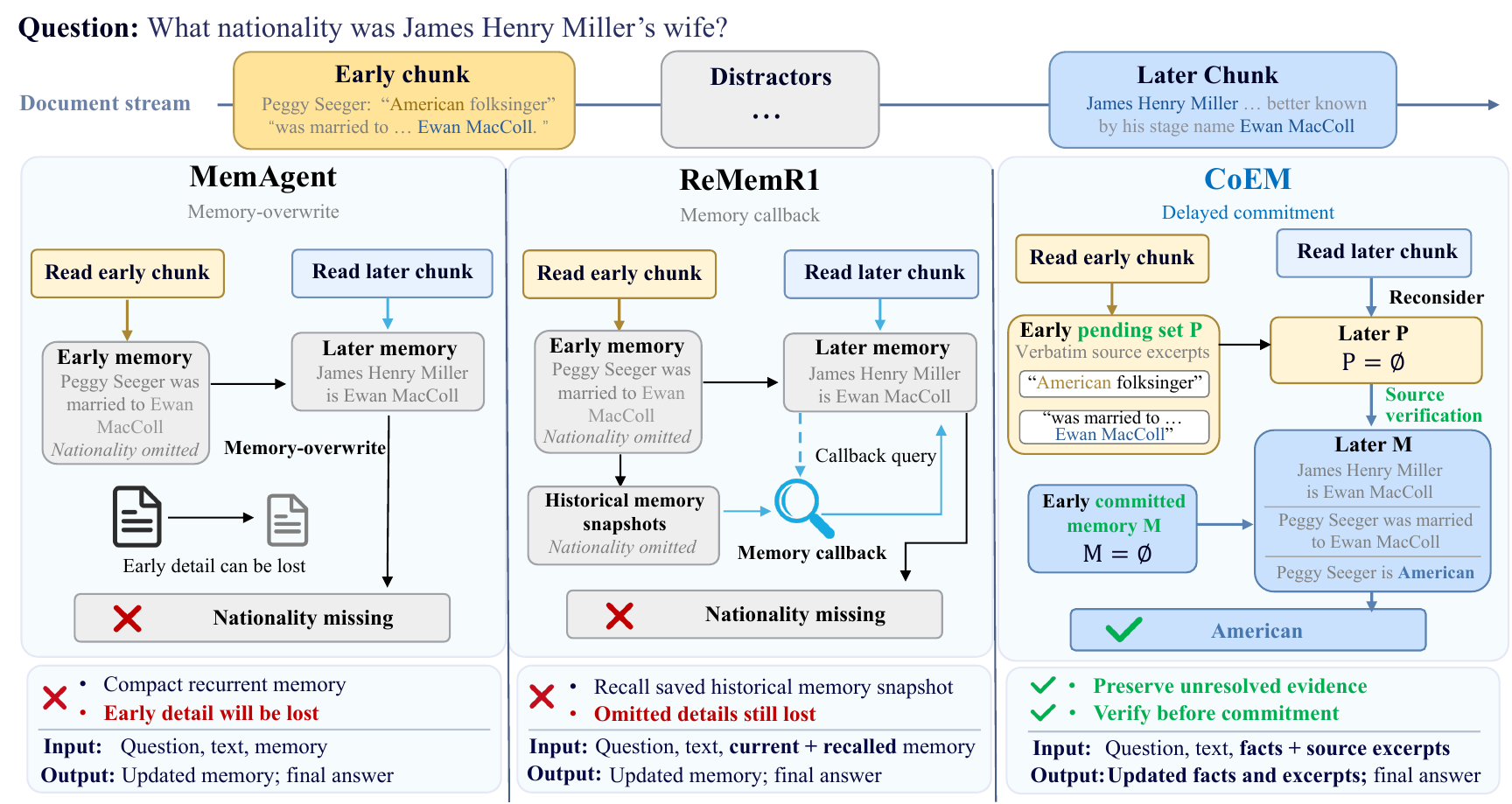}
\vspace{-7mm}
\caption{
Methodology comparison of \our and two baselines on a HotpotQA question~\citep{yang2018hotpotqa}. 
\our does not make a decision about the unclear source evidence until its value is clarified by the later evidence. 
}
\label{fig:limitation}
\vspace{-6mm}
\end{figure}

Our diagnostic analysis across three multi-hop question-answering benchmarks finds that 32.03\%--44.53\% of random sampled questions exhibit delayed evidence relevance (details are provided in Appendix~\ref{app:retention-diagnostics}).
This prevalence highlights the pressing need to preserve evidence whose importance becomes apparent only after later context is observed. Meanwhile, as shown in Figure~\ref{fig:memory-occupancy}, MemAgent occupies substantially more memory than GRU-Mem yet achieves lower question-answering performance. This contrast suggests room to retain more useful evidence within the same memory budget through careful decisions about what to preserve and when to compress it. Motivated by these observations, we introduce Commit-on-Evidence Memory (\our), a recurrent memory agent that learns when to commit evidence under partial observability of long context and evidence.

\begin{wrapfigure}{r}{0.5\textwidth}
\centering
\vspace{-4mm}
\includegraphics[width=\linewidth]{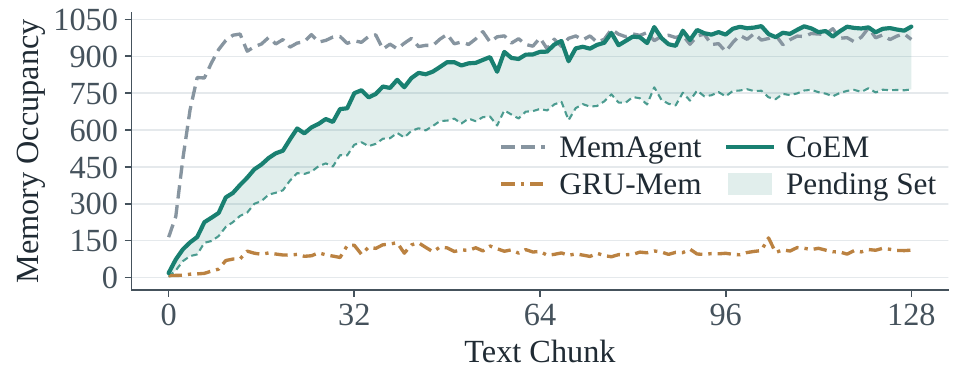}
\vspace{-9mm}
\caption{Memory occupancy of different methods throughout reading procedure on HotpotQA.}
\label{fig:memory-occupancy}
\vspace{-4mm}
\end{wrapfigure}

As shown in the right part of Figure~\ref{fig:limitation}, within a fixed context-memory budget, \our maintains a committed memory of compact facts and a pending set of short, verbatim source excerpts, both of which are carried into subsequent model contexts. As each chunk arrives, the policy revisits pending evidence and chooses to \textsc{Promote}, \textsc{Keep}, or \textsc{Drop} each evidence. \textsc{Keep} preserves the evidence until later observations clarify its relevance, whereas \textsc{Promote} converts it into a compact fact and commits the fact to the committed memory. Before the commitment, a frozen natural language inference (NLI) verifier checks whether the corresponding excerpts entail the promoted fact. 
We train the policy utilizing reinforcement learning strengthened by dense step-level evidence rewards together with the final answer rewards. 
In this way, the memory agent learns to coordinate evidence retention and verification, having knowledge about when and how to commit true evidence.

We evaluate \our using Qwen3.5-4B and Qwen3.5-9B backbones on HotpotQA, 2WikiMultiHopQA, and MuSiQue. Across both model scales and all tested context lengths, \our consistently outperforms prior recurrent-memory agents. On the longest input contexts of 6,400 documents, \our with Qwen3.5-9B improves answer F1 by 10.4--11.4 points over the strongest baseline under the same memory budget. 
\our also remains efficient in practice by processing context sequentially with bounded evidence pending set and using a lightweight NLI verifier on short source excerpts.
Overall, our key contributions are summarized as follows:

\vspace{-3mm}
\begin{itemize}
    \item  We identify a fundamental limitation of existing recurrent memory agents: premature information compression can irreversibly discard evidence that appears irrelevant when first observed but later becomes essential for reasoning.
    
    \item We introduce \our, which partitions a fixed context memory budget into committed memory and pending set. As new chunk arrives, a trained policy promotes, retains, or discards pending evidence, and a frozen verifier only admits source-supported facts.

    \item  We evaluate \our with both Qwen3.5-4B and Qwen3.5-9B backbones on three multi-hop question-answering benchmarks across context lengths ranging from 50 to 6,400 documents. \our consistently outperforms existing recurrent memory methods under the same context-memory budget while maintaining computational efficiency.

\end{itemize}


\section{Related Work}
\label{sec:related-work}

\paragraph{Memory Mechanisms for LLMs.}
Memory-augmented LLM systems extend their working context by storing and retrieving past information. Retrieval-augmented generation uses external corpora as non-parametric memory~\citep{wang-etal-2024-searching}. MemoryBank maintains conversational memories with selective forgetting~\citep{zhong2023memorybankenhancinglargelanguage}, and MemGPT manages information across hierarchical memory tiers~\citep{packer2024memgptllmsoperatingsystems}. ReadAgent constructs gist memories and reopens linked source passages~\citep{3692070.3693124}, whereas LightMem separates short-term organization from long-term consolidation~\citep{ICLR2026_a05b7265}. These approaches largely rely on external stores to retain information beyond the active context. In contrast, \our carries all retained information in the strictly bounded recurrent textual memory.

\vspace{-3mm}
\paragraph{Recurrent Memory for Long-Context Reasoning.}
Recurrent memory agents process inputs chunk by chunk while carrying a bounded textual memory forward. MemAgent learns an RL-based memory-overwrite policy to process inputs beyond its native context length~\citep{yu2026memagent}. ReMemR1 adds callbacks over historical memory snapshots~\citep{shi2025rememr1}, while GRU-Mem uses update and exit gates to skip unnecessary writes and terminate reading early~\citep{sheng2026grumem}. These mechanisms improve memory reuse and efficiency, but cannot recover source details discarded before their relevance becomes clear. \our retains unresolved evidence in a bounded pending set, allowing later context to clarify their relevance before they are processed.

\vspace{-3mm}
\paragraph{Reinforcement Learning for Memory Management.}
Final answer rewards provide limited supervision for intermediate memory decisions. ReMemR1 supplements outcome rewards with step-level feedback for memory updates and callbacks~\citep{shi2025rememr1}. LongRLVR introduces verifiable context rewards for evidence selection~\citep{chen2026longrlvr}, while InfoMem evaluates final-memory utility through answer-conditioned information gain~\citep{han2026infomem}. \our combines final-answer rewards with step-level evidence rewards for source-supported memory proposals and valid operations. A frozen verifier checks proposed facts against available source excerpts, while the evidence reward propagates subsequent signals to earlier retention and commitment decisions.

\section{\our: Memory Control under Partial Observability}
\label{sec:method}

Figure~\ref{fig:components} presents an overview of \our, which retains unresolved evidence for future decisions and uses a frozen verifier to check source support before commitment. Its memory management policy is trained through reinforcement learning with step-level evidence rewards as shown in Figure~\ref{fig:training}.

\begin{figure}[t]
\centering
\includegraphics[width=0.9\linewidth]{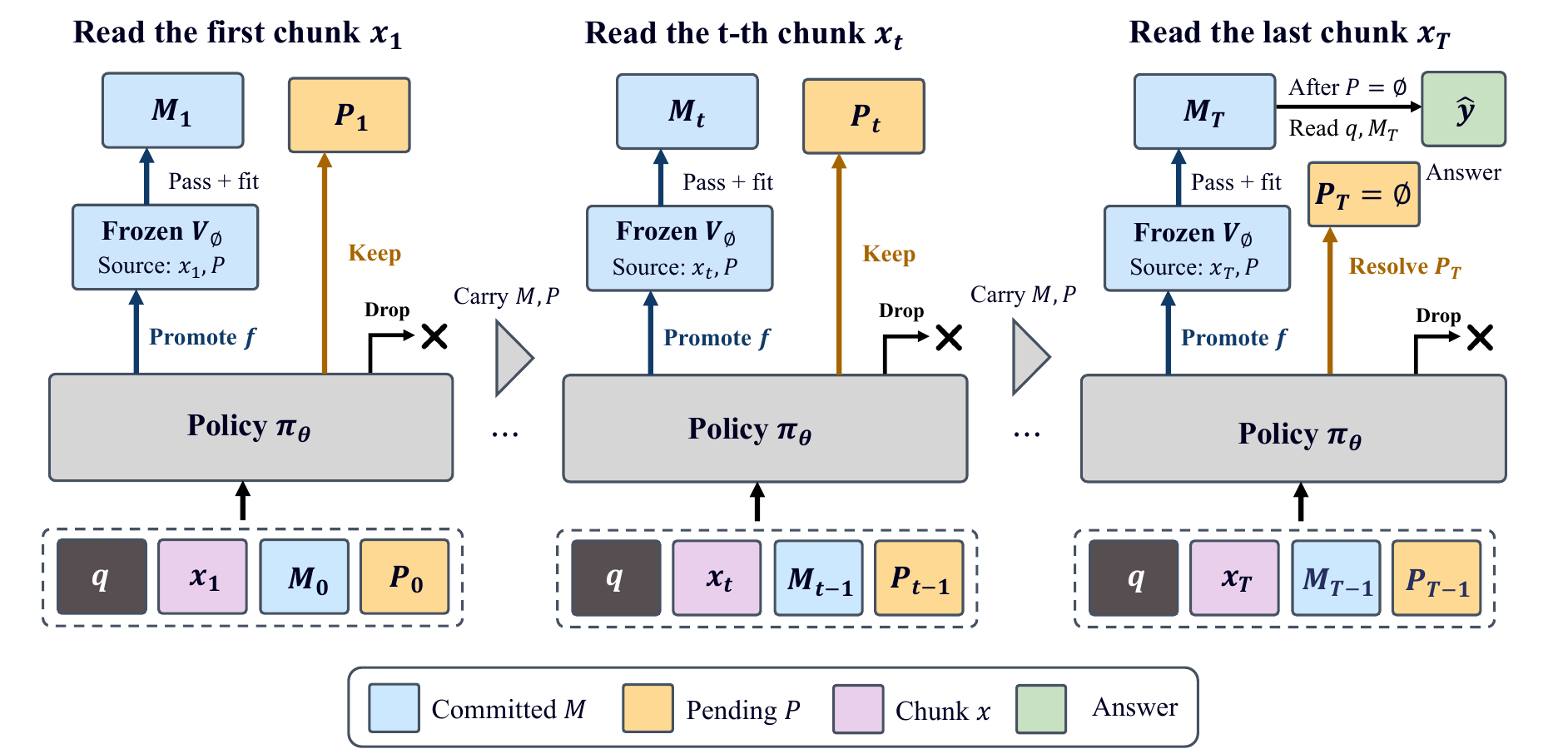}
\vspace{-3mm}
\caption{Overview of \our. At intermediate steps $t<T$, the policy promotes, keeps, or drops evidence. A frozen verifier checks proposed facts before commitment. At the final step $T$, remaining pending evidence is resolved before the model answers from committed memory.
}
\label{fig:components}
\vspace{-5mm}
\end{figure}

\subsection{Problem Setting and Context Memory}
\label{sec:state}

Following MemAgent~\citep{yu2026memagent}, given a question $q$ and a chunked document stream $D=(x_1,\ldots,x_T)$, the model reads chunks of at most $C$ tokens step by step and generates an answer after reading the final chunk. We formulate this process as a
finite-horizon POMDP~\citep{kaelbling1998pomdp}: future chunks are unobserved, and historical chunks cannot be revisited. 
Unlike conventional recurrent memory agents, \our transforms their context memory into a bounded state $h_t=(M_t,P_t)$ which is carried across reading steps, 
where the committed memory $M_t$ stores compact committed facts and the pending set $P_t$ retains verbatim source excerpts whose relevance remains unresolved. 
At step $t$, the policy observes the question $q$, the current chunk $x_t$, and the preceding state $h_{t-1}$, and then samples an action $a_{t}$ that updates both committed memory and pending set:
\begin{equation}
\label{eq:controller}
\small
\begin{aligned}
\mathcal{C}_t
    &\coloneqq \left(q, x_t, h_{t-1}\right)
     = \left(q, x_t, M_{t-1}, P_{t-1}\right), \\
a_t
    &\sim \pi_\theta\!\left(\,\cdot \mid \mathcal{C}_t\right), \\
h_t
    &\coloneqq \left(M_t, P_t\right)
     = F\!\left(h_{t-1}, x_t, a_t\right).
\end{aligned}
\end{equation}
where $\mathcal{C}_t$ denotes the context visible to the policy $\pi_\theta$, parameterized by $\theta$, at step $t$. The step-level action $a_t$ comprises all memory decisions, including any operations that reconsider pending evidence in $P_{t-1}$. The \our{} controller $F$ is a deterministic, non-learned procedure that executes these decisions with source verification and budget checks, producing the updated state $h_t=(M_t,P_t)$. This bounded state is the only information from previous chunks $x_1,\ldots,x_t$ carried into step $t+1$; Appendix~\ref{sec:theory-pomdp} formalizes its relationship to the full POMDP state.

The two components of $h_t$ serve complementary roles. Committed memory $M_t$ contains entries $e=(f,\rho)$, where $f$ is a concise, source-supported fact and $\rho$ identifies its supporting source spans. The pending set $P_t$ contains evidence $p=(s,\rho,b)$, where $s$ is a
verbatim source excerpt and $b$ is its current rank in admission order. Ranks are renumbered whenever pending evidence are removed from $P_t$. In both $M_t$ and $P_t$, $\rho$ records the source chunk and token offsets for provenance; it cannot recover the text from previous chunks or removed evidence.

The budgets of $h_t$, $M_t$, and $P_t$ are strictly bounded. Let $\ell$ count serialized backbone tokens, including text, identifiers, ranks, and separators.
Both $M_t$, and $P_t$ must remain within their allocated fixed budget:
\begin{equation}
\small
 B_M+B_P=B,\qquad \ell(M_t)\leq B_M,\qquad
 \ell(P_t)\leq B_P,\qquad \ell(h_t)\leq B,
 \label{eq:budget}
\end{equation}
where $B$ is the total context memory budget same as conventional memory agents, and $B_M$ and $B_P$ are the portions allocated to committed facts and pending set, respectively. Any formatting overhead in $h_t$ is charged to $B_M$, ensuring that all persistent tokens are included in the budget.

\subsection{Unresolved Evidence Management}
\label{sec:defer}

Retaining unresolved evidence allows the policy to defer decisions until later context clarifies its relevance. At each reading step $t\in\{1,\ldots,T\}$, \our processes the current chunk $x_t$ starting from the preceding state $(M_{t-1},P_{t-1})$. For evidence from $x_t$ selected for retention, the policy specifies source offsets, and \our copies the corresponding tokens from $x_t$ into the pending set $P_t$.

During step $t$, the policy revisits evidence carried in $P_{t-1}$ using information from the current chunk $x_t$. For each reconsideration of pending evidence $p$, the policy selects an action from
\begin{equation}
\small
 \mathcal{A}
 =\{\textsc{Promote},\textsc{Keep},\textsc{Drop}\}.
 \label{eq:actions}
\end{equation}
\textsc{Promote} uses $p$ to propose a compact fact for committed memory $M_t$. A successful promotion updates committed memory $M_t$ before its source excerpt $s$ is released.
\textsc{Keep} preserves $p$ unchanged in the pending set $P_t$, while \textsc{Drop} removes $p$ and releases occupied capacity of $P_t$.
A promotion is accepted only after passing the source verification and memory capacity checks introduced in Section~\ref{sec:verification}.
For evidence $p$ admitted to the pending set $P_t$, let $t_{\mathrm{arr}}(p)$ denote its admission step.
We define its commitment time as
\begin{equation}
\small
 t_p
 =\inf\left\{
 t\in\{t_{\mathrm{arr}}(p),\ldots,T\}:
 \text{a promotion using }p\text{ is accepted during step }t
 \right\},
 \label{eq:stopping-time}
\end{equation}
where $t_p=\infty$ if $p$ is never committed during the reading process.

At an intermediate step $t<T$, \textsc{Keep} allows evidence to remain pending without an age limit, subject to the pending set budget $B_P$.
If admitting new evidence would exceed $B_P$, \our allows bounded reconsideration to release capacity.
If space remains insufficient, \our repeatedly resolves the oldest pending evidence through successful promotion or dropping until the new evidence fits.
Appendix~\ref{sec:algorithm-details} specifies candidate validation and the limits on reconsideration and promotion attempts.
The resulting state $(M_t,P_t)$ is then carried into step $t+1$.

At the final step $t=T$, \our processes $x_T$ using the same procedure, then resolves all remaining pending evidence before completing the step.
Each remaining item must be successfully promoted or dropped, with \textsc{Keep} disabled.
The final pending set $P_T$ is therefore empty, and the model generates its answer $\hat y$ from the question $q$ and final committed memory $M_T$:
\begin{equation}
\small
 P_T=\varnothing,\qquad \hat y\sim\pi_\theta(\cdot\mid q,M_T),
 \label{eq:answer}
\end{equation}

\subsection{Source-Verified Commitment and Atomic Memory Updates}
\label{sec:verification}

For each proposed entry $e=(f,\rho)$ at step $t$, \our collects the cited source spans still available in $x_t$ or the pending set $P_t$. A fact combining information from multiple passages must cite all excerpts required to support it. These verbatim excerpts are concatenated into a premise $S_e$. 
A frozen NLI classifier $V_\phi$ then evaluates whether $S_e$ entails the proposed fact $f$:
\begin{equation}
\small
 v_e
 =V_\phi(\mathrm{entailment}\mid S_e,f),
 \qquad
 g(e,S_e)
 =\mathbf 1[v_e\geq\eta]\,
  \mathbf 1[\mathrm{valid}(e,S_e)],
 \label{eq:gate}
\end{equation}
where $v_e$ is the entailment probability produced by the verifier with frozen parameters $\phi$, $\eta$ is the acceptance threshold, and $g(e,S_e)\in\{0,1\}$ is the resulting verification gate. 
$\mathrm{valid}(e,S_e)$ checks that 
all cited spans remain available, and the cited text exactly matches the original source.
If a promotion is rejected, committed memory and its supporting pending excerpts remain unchanged unless capacity occupancy of the pending set needs to be released.
We instantiate $V_\phi$ as a DeBERTaV3-small NLI cross-encoder~\citep{he2023debertav3} and keep it frozen throughout policy optimization with more details in Appendix~\ref{sec:verifier-details}. The verifier determines whether a proposed fact is supported by its cited sources.

To create space in committed memory, the policy may specify a set of whole entries $E^-\subseteq M$ for removal and a set of new entries $E^+$ for insertion. \our treats the proposed replacement as an atomic transaction: it is applied only if every new entry passes verification and the resulting memory satisfies the committed-memory
budget. Formally, committed memory $M$ is updated as follows:
\vspace{-2mm}
\begin{equation}
\small
 \widetilde M=(M\setminus E^-)\cup E^+,\qquad
 M'=
 \begin{cases}
 \widetilde M,
 & \ell(\widetilde M)\leq B_M
   \ \land\
   \displaystyle\bigwedge_{e\in E^+} g(e,S_e)=1,\\
 M,
 & \text{otherwise},
 \end{cases}
 \label{eq:transaction}
\end{equation}
where $\widetilde M$ is the proposed replacement and $M'$ is the committed memory retained after validation. The complete serialization of $M'$ and $P$ must also satisfy the global constraint in Equation~(\ref{eq:budget}). If any verification or capacity check fails, the entire transaction is rejected, leaving all existing entries
unchanged. Rewriting or merging an existing fact is treated as the insertion of a new statement and therefore requires renewed source verification. If the necessary source excerpts are no longer available in $x_t$ or $P$,
the existing fact may only be preserved unchanged or deleted.

\subsection{Step-Level Evidence Reward for Learning Delayed Memory Decisions}
\label{sec:training}

We train the policy with a combination of a trajectory-level answer reward and a step-level evidence reward.
Our step-level evidence reward evaluates source support and operation validity, with a return-to-go that carries subsequent verification feedback to earlier memory decisions.
For each training sample $(q,D,y)$, the rollout policy $\pi_{\theta_{\mathrm{old}}}$, which is frozen during trajectory collection, samples $G$ complete trajectories $\{\tau_i\}_{i=1}^{G}$.
The reference answer $y$ is used only for reward computation, giving the final answer reward
\begin{equation}
\small
 R_i^{\mathrm{ans}}=r_{\mathrm{ans}}(\hat y_i,y)\in[0,1],
\end{equation}
where $r_{\mathrm{ans}}$ compares the generated answer $\hat y_i$ from the policy with $y$.

\paragraph{Step-Level Evidence Reward.}
At step $t$, $m_{i,t}^{\mathrm{inv}}$ counts invalid operations, including violations of action format, source availability or matching, and memory constraints.
Among proposals passing all non-entailment checks, $m_{i,t}^{\mathrm{rej}}$ counts facts rejected by the frozen verifier because $v_e<\eta$.
Invalid operations contribute only to $m_{i,t}^{\mathrm{inv}}$, in order to avoid double counting.
We define the step-level evidence reward and its return-to-go as
\vspace{-2mm}
\begin{equation}
\small
 r_{i,t}^{\mathrm{evd}}
 =-\lambda_v m_{i,t}^{\mathrm{rej}}
  -\lambda_f m_{i,t}^{\mathrm{inv}},
 \qquad
 L_{i,t}=\frac{1}{T}\sum_{u=t}^{T}r_{i,u}^{\mathrm{evd}},
 \label{eq:reward}
\end{equation}
where $\lambda_v,\lambda_f\geq0$ weight the two failure types, and $1/T$ normalizes the return by the document's total number of chunks.
Feedback at step $u$ enters every return $L_{i,t}$ with $t\leq u$, allowing later verification outcomes to influence earlier admission and \textsc{Keep} decisions.
Accepted committed facts receive no per-write bonus and the final answer rewards guide the selection of task-relevant facts.

At each step, the rollouts have read the same document prefix but may retain different committed memories and pending sets.
We center answer rewards and evidence returns within the rollout group:
\begin{equation}
 \small
 A_i^{\mathrm{ans}}
 =R_i^{\mathrm{ans}}-\frac{1}{G}\sum_{j=1}^{G}R_j^{\mathrm{ans}},
 \qquad
 A_{i,t}^{\mathrm{evd}}
 =L_{i,t}-\frac{1}{G}\sum_{j=1}^{G}L_{j,t},
 \label{eq:local_advantage}
\end{equation}
and combine the resulting advantages as
\begin{equation}
\small
 A_{i,t}
 =\alpha A_i^{\mathrm{ans}}+(1-\alpha)A_{i,t}^{\mathrm{evd}},
 \qquad 0\leq\alpha\leq1,
 \label{eq:combined_advantage}
\end{equation}
where $\alpha$ balances answer quality and evidence validity.
During answer generation, $L_{i,T+1}=0$, so $A_{i,T+1}=\alpha A_i^{\mathrm{ans}}$.

\begin{figure}[t]
\centering
\includegraphics[width=\linewidth]{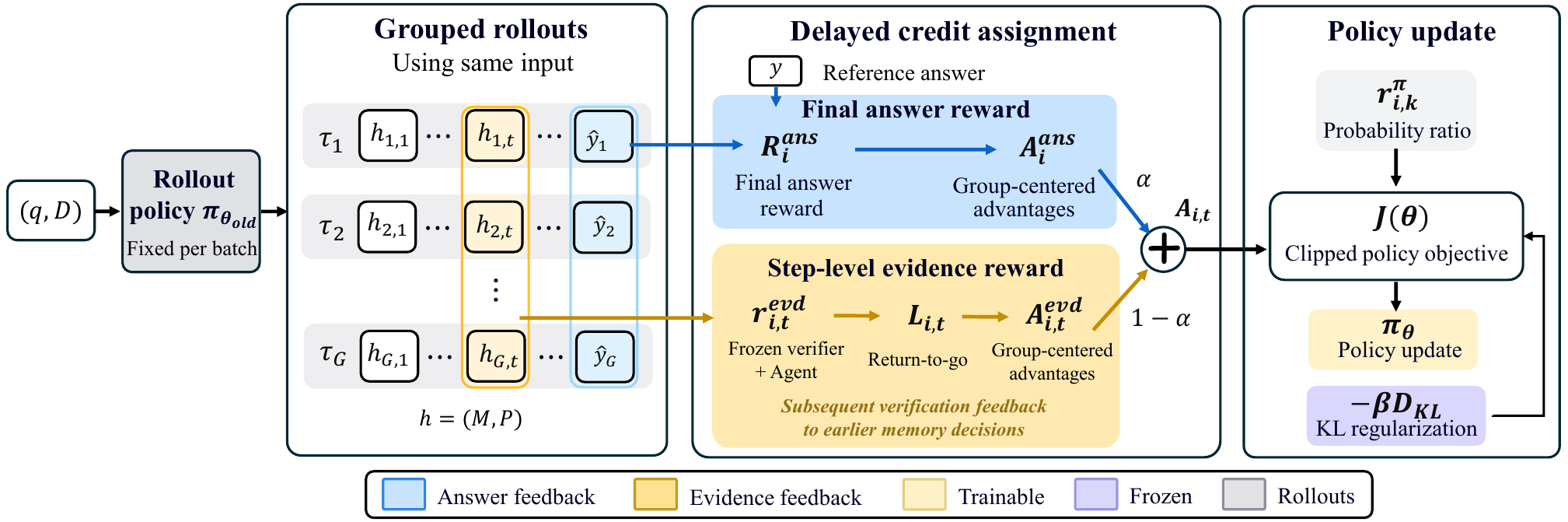}
\vspace{-5mm}
\caption{RL Training with step-level evidence rewards. Grouped rollouts combine final answer rewards with evidence return-to-go, carrying later source-verification feedback to earlier memory decisions. Group-centered advantages drive the clipped policy update.}
\label{fig:training}
\vspace{-3mm}
\end{figure}

\paragraph{Policy Optimization.}
We apply the combined advantages through a clipped policy objective~\citep{schulman2017ppo}.
Let $\mathcal I_i$ index policy-generated tokens $w_{i,k}$ with conditioning contexts $\xi_{i,k}$, and let $N_{\mathrm{gen}}=\sum_i|\mathcal I_i|$.
The index $t(i,k)$ denotes the reading step, or $T+1$ for answer generation, and all policy operations within a chunk share its advantage.
We maximize $\mathcal J$ as follows:
\vspace{-1mm}
\begin{equation}
\small
 \begin{aligned}
 \mathcal J(\theta)
 &=\mathbb E\!\left[
 \frac1{N_{\mathrm{gen}}}
 \sum_{i=1}^G\sum_{k\in\mathcal I_i}
 \ell_\epsilon\!\left(r^\pi_{i,k}(\theta),A_{i,t(i,k)}\right)
 -\beta\mathcal D_{\mathrm{KL}}(\theta)\right],\\
 \ell_\epsilon(r,A)
 &=\min\{rA,\operatorname{clip}(r,1-\epsilon,1+\epsilon)A\},\\
 r^\pi_{i,k}(\theta)
 &=\frac{\pi_\theta(w_{i,k}\mid\xi_{i,k})}
 {\pi_{\theta_{\mathrm{old}}}(w_{i,k}\mid\xi_{i,k})},\\
 \mathcal D_{\mathrm{KL}}(\theta)
 &=\frac1{N_{\mathrm{gen}}}
 \sum_{i=1}^G\sum_{k\in\mathcal I_i}
 D_{\mathrm{KL}}\!\left(
 \pi_\theta(\cdot\mid\xi_{i,k})
 \,\|\,
 \pi_{\mathrm{ref}}(\cdot\mid\xi_{i,k})\right),
 \end{aligned}
 \label{eq:clipped_objective}
\end{equation}
where $\epsilon$ is the clipping width, $\pi_{\mathrm{ref}}$ is the frozen initial reference policy, and $\beta\geq0$ weights KL regularization.
The expectation averages over training inputs and sampled rollout groups. The detailed definition of $D_{\mathrm{KL}}$ can be found in Appendix~\ref{sec:algorithm-details}.
Only policy-generated tokens receive policy gradients while source text, \our agent execution, and frozen-verifier outputs do not.
Appendix~\ref{sec:theory} analyzes evidence retention and the learning signal from verification feedback.

\section{Experiments}
\label{sec:experiments}
\subsection{Experimental Setup}
\label{sec:experimental-setup}
\paragraph{Benchmarks and Models.}
We evaluate \our on three multi-hop question-answering benchmarks using Qwen3.5-4B and Qwen3.5-9B~\citep{qwen2026qwen35fourb,qwen2026qwen35nineb}.
HotpotQA~\citep{yang2018hotpotqa} provides the training data and in-distribution (ID) evaluation, while 2WikiMultiHopQA~\citep{ho2020twowiki} and MuSiQue~\citep{trivedi2022musique} provide out-of-distribution (OOD) evaluation without additional training. Following ReMemR1~\citep{shi2025rememr1}, we sample 128 held-out questions from each dataset and evaluate the same questions at eight context lengths ranging from 50 to 6,400 documents. 
Moreover, we adopt the normalized token-level answer F1 as the primary metric (see Appendix~\ref{app:data-construction} for details).

\vspace{-3mm}
\paragraph{Baselines and Evaluation Protocol.}
We compare \our{} with general-purpose LLMs and recurrent memory agents: MemAgent~\citep{yu2026memagent}, GRU-Mem~\citep{sheng2026grumem}, and ReMemR1~\citep{shi2025rememr1}. All memory agents share the same training setup and 1,024-token budget, which for ReMemR1 covers its current memory, one historical snapshot, and callback query. Its native archive-enabled configuration still trails \our{} by 2.3--4.3 average F1 points (Table~\ref{tab:native-archive}).

\begin{table}[t]
\centering
\caption{
Answer F1 (\%) across context lengths. $\dagger$ indicates that all retained ReMemR1 memories and queries share the 1,024-token budget. \texttt{--} indicates that the length of input exceeds the context window length of the model. Bold denotes the best result for each backbone and document count.
}
\label{tab:main}
\begingroup
\scriptsize
\setlength{\tabcolsep}{2.6pt}
\renewcommand{\arraystretch}{1.01}

\begin{NiceTabular*}{\linewidth}{
  @{\extracolsep{\fill}}
  ll
  *{8}{r}
  @{\hspace{5pt}}
  r
  @{}
}
\CodeBefore
  \rowcolor{black!8}{3,14,25}
\Body
\toprule
\multirow{2}{*}{Backbone}
& \multirow{2}{*}{Method}
& \multicolumn{8}{c}{Number of documents}
& \multirow{2}{*}{Avg.} \\
\cmidrule(lr){3-10}
& & 50 & 100 & 200 & 400 & 800 & 1,600 & 3,200 & 6,400 & \\
\midrule

\multicolumn{11}{c}{\textbf{HotpotQA}\enspace\textit{(in-distribution)}} \\
\multirow{5}{*}{Qwen3.5-4B}
& Full context
& 83.3 & 82.8 & 79.7 & 77.2 & 71.6 & 68.4 & -- & -- & -- \\

& MemAgent
& 83.1 & 83.0 & 81.8 & 80.6 & 79.2 & 75.7 & 74.5 & 71.9 & 78.7 \\

& GRU-Mem
& 84.4 & 83.5 & 83.8 & 80.9 & 81.4 & 78.2 & 77.0 & 75.4 & 80.6 \\

& ReMemR1$^\dagger$
& 86.2 & 84.2 & 83.3 & 81.9 & 81.3 & 79.8 & 78.5 & 74.3 & 81.2 \\

& \Block[fill=CoEMBlue]{1-10}{}\textbf{CoEM}
& \textbf{87.0} & \textbf{86.7} & \textbf{86.3} & \textbf{85.9} & \textbf{85.5} & \textbf{85.1} & \textbf{84.7} & \textbf{84.2} & \textbf{85.7} \\

\cmidrule(lr){1-11}
\multirow{5}{*}{Qwen3.5-9B}
& Full context
& 86.2 & 85.2 & 83.3 & 80.0 & 75.3 & 71.8 & -- & -- & -- \\

& MemAgent
& 87.8 & 87.2 & 84.5 & 83.4 & 82.1 & 79.0 & 78.0 & 75.2 & 82.2 \\

& GRU-Mem
& 88.8 & 86.2 & 85.4 & 85.0 & 83.9 & 82.9 & 81.0 & 78.6 & 84.0 \\

& ReMemR1$^\dagger$
& 88.3 & 87.0 & 86.7 & 84.9 & 84.7 & 83.3 & 80.3 & 78.7 & 84.2 \\

& \Block[fill=CoEMBlue]{1-10}{}\textbf{CoEM}
& \textbf{90.1} & \textbf{90.0} & \textbf{89.9} & \textbf{89.6} & \textbf{89.4} & \textbf{89.3} & \textbf{89.2} & \textbf{89.1} & \textbf{89.6} \\

\midrule

\multicolumn{11}{c}{\textbf{2WikiMultiHopQA}\enspace\textit{(out-of-distribution)}} \\
\multirow{5}{*}{Qwen3.5-4B}
& Full context
& 71.1 & 71.0 & 68.1 & 65.3 & 62.1 & 55.7 & -- & -- & -- \\

& MemAgent
& 72.6 & 71.2 & 69.4 & 65.9 & 63.1 & 60.8 & 58.2 & 54.6 & 64.5 \\

& GRU-Mem
& 73.9 & 72.7 & 72.8 & 70.9 & 68.3 & 65.5 & 64.4 & 61.1 & 68.7 \\

& ReMemR1$^\dagger$
& 76.5 & 75.6 & 72.9 & 71.9 & 71.3 & 67.8 & 66.8 & 64.4 & 70.9 \\

& \Block[fill=CoEMBlue]{1-10}{}\textbf{CoEM}
& \textbf{78.5} & \textbf{78.0} & \textbf{77.3} & \textbf{76.8} & \textbf{76.1} & \textbf{75.5} & \textbf{74.7} & \textbf{73.9} & \textbf{76.4} \\

\cmidrule(lr){1-11}
\multirow{5}{*}{Qwen3.5-9B}
& Full context
& 77.2 & 76.5 & 72.9 & 70.7 & 67.1 & 60.3 & -- & -- & -- \\

& MemAgent
& 77.7 & 76.6 & 73.8 & 72.4 & 68.5 & 65.1 & 62.9 & 59.6 & 69.6 \\

& GRU-Mem
& 79.5 & 78.0 & 77.2 & 75.9 & 74.4 & 71.6 & 69.9 & 66.2 & 74.1 \\

& ReMemR1$^\dagger$
& 80.9 & 80.3 & 79.0 & 76.3 & 74.5 & 73.0 & 70.8 & 69.6 & 75.6 \\

& \Block[fill=CoEMBlue]{1-10}{}\textbf{CoEM}
& \textbf{83.0} & \textbf{82.7} & \textbf{82.4} & \textbf{81.8} & \textbf{81.5} & \textbf{81.3} & \textbf{81.1} & \textbf{81.0} & \textbf{81.9} \\

\midrule

\multicolumn{11}{c}{\textbf{MuSiQue}\enspace\textit{(out-of-distribution)}} \\
\multirow{5}{*}{Qwen3.5-4B}
& Full context
& 63.7 & 62.8 & 61.5 & 58.7 & 52.5 & 47.5 & -- & -- & -- \\

& MemAgent
& 62.7 & 61.4 & 60.1 & 58.4 & 56.5 & 52.8 & 52.1 & 49.2 & 56.7 \\

& GRU-Mem
& 64.2 & 64.6 & 62.0 & 61.1 & 58.9 & 57.9 & 56.2 & 52.3 & 59.7 \\

& ReMemR1$^\dagger$
& 65.9 & 64.5 & 63.8 & 63.0 & 60.4 & 59.7 & 57.9 & 55.9 & 61.4 \\

& \Block[fill=CoEMBlue]{1-10}{}\textbf{CoEM}
& \textbf{69.1} & \textbf{68.7} & \textbf{68.1} & \textbf{67.5} & \textbf{66.8} & \textbf{66.1} & \textbf{65.3} & \textbf{64.5} & \textbf{67.0} \\

\cmidrule(lr){1-11}
\multirow{5}{*}{Qwen3.5-9B}
& Full context
& 70.9 & 68.6 & 67.0 & 64.4 & 58.8 & 52.3 & -- & -- & -- \\

& MemAgent
& 69.3 & 67.0 & 67.1 & 65.0 & 61.2 & 58.8 & 56.9 & 55.5 & 62.6 \\

& GRU-Mem
& 71.7 & 69.6 & 68.8 & 67.9 & 65.5 & 62.9 & 61.5 & 58.3 & 65.8 \\

& ReMemR1$^\dagger$
& 72.6 & 70.3 & 70.8 & 68.0 & 66.8 & 64.3 & 63.3 & 61.6 & 67.2 \\

& \Block[fill=CoEMBlue]{1-10}{}\textbf{CoEM}
& \textbf{74.4} & \textbf{74.2} & \textbf{74.0} & \textbf{73.7} & \textbf{73.6} & \textbf{73.3} & \textbf{73.0} & \textbf{72.3} & \textbf{73.6} \\

\bottomrule
\end{NiceTabular*}
\endgroup
\vspace{-5mm}
\end{table}

\vspace{-4mm}
\paragraph{Implementation Details.}
All recurrent memory agents process $C=5{,}000$-token chunks at each reading step and have the same $B=1{,}024$-token context-memory budget. \our allocates $B_P=256$ tokens to pending set $P$ and $B_M=768$ tokens to committed memory $M$, and uses a frozen DeBERTaV3-small NLI verifier with threshold $\eta=0.90$. All memory agents are trained on the same set of 32,768 HotpotQA samples, each paired with 200 documents, using eight rollouts per question. Both models are trained on eight NVIDIA H800 GPUs.
During evaluation, all memory agents use a 9,216-token serving window, with up to 8,192 input tokens and 1,024 output tokens, and the final answers are capped at 256 tokens. Further implementation details are provided in Appendices~\ref{app:implementation}
and~\ref{app:compute-resources}.

\subsection{Main Results}
\label{sec:main-results}
Table~\ref{tab:main} shows that \our consistently outperforms existing methods in all experimental settings. Its advantage becomes increasingly pronounced as the length of the input text increases. With the Qwen3.5-9B backbone, the margin over the strongest memory baseline increases from 1.3--2.1 F1 points at 50 documents to 10.4--11.4 F1 points at 6,400 documents. Across the range of 50 to 6,400 documents, \our's answer F1 declines by only 1.0--2.1 points, compared to a 9.6--11.3-point decline for ReMemR1. This widening gap suggests that preserving unresolved evidence and source-verified commitment become increasingly important as input length increases.

\our is also robust to long, distractor-heavy inputs. The performance of the genral-purpose LLMs decreases sharply as the document count grows, and inputs with more than 1,600 documents exceed their configured context window. At 1,600 documents, \our with the Qwen3.5-9B backbone outperforms the genral-purpose LLMs by 17.5 F1 points on HotpotQA and 21.0 F1 points on both 2WikiMultiHopQA and MuSiQue.
Moreover, the trained \our policy generalizes beyond its training distribution. Without additional training, \our performs best at every input length on 2WikiMultiHopQA and MuSiQue. The same trend holds for Qwen3.5-4B, suggesting that \our learns a generalizable policy for managing source evidence across datasets and model scales.

\subsection{More Analysis}
\label{sec:diagnostics}

\begin{figure}[t]
    \centering
    \includegraphics[width=0.9\linewidth]{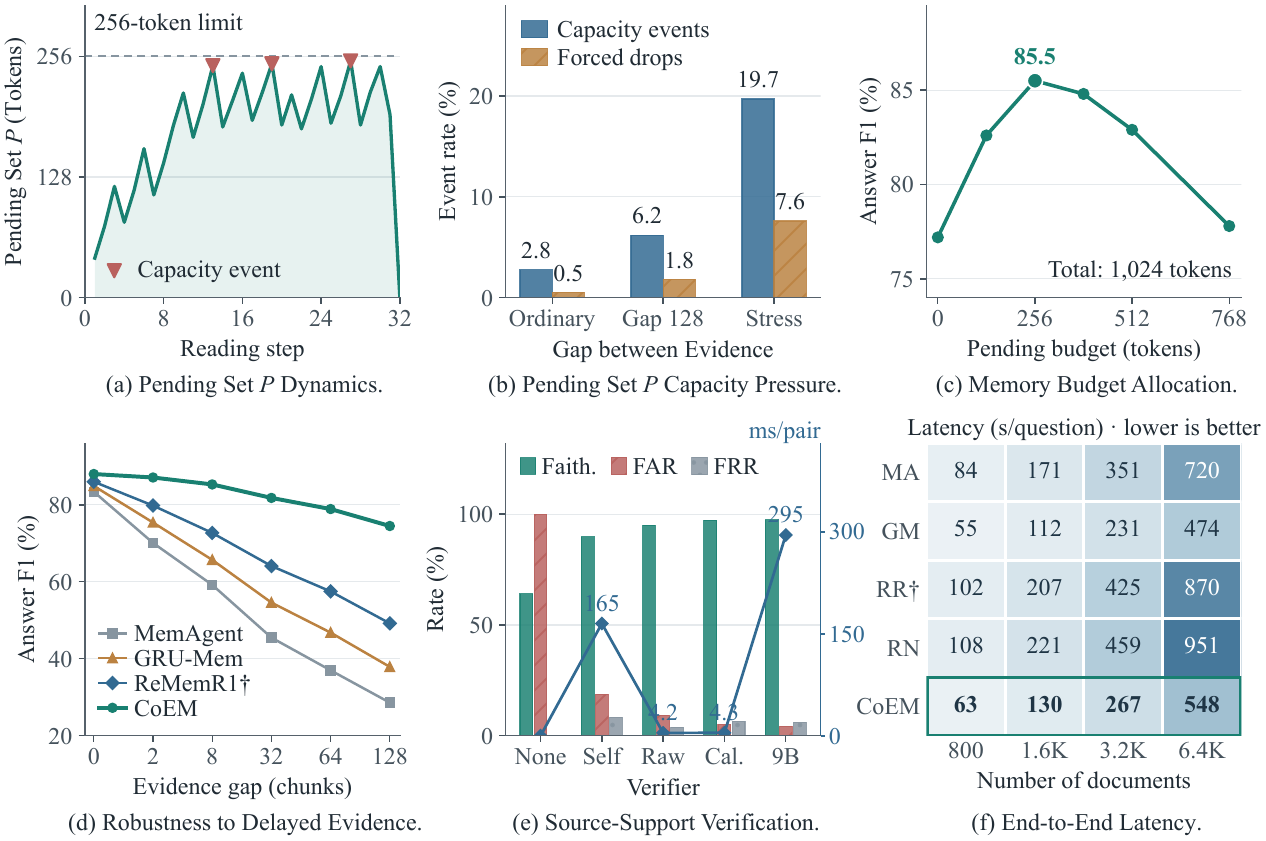}
    \vspace{-2mm}
    \caption{Analysis of evidence retention, source verification, and inference efficiency with Qwen3.5-4B.
    In (a), Capacity events occur when admitting new evidence would exceed the pending-set budget.
    In (b), Forced drops discard the oldest pending evidence to free space.
    In (e), Self is the 4B policy self-check, Raw/Cal. are the raw/calibrated DeBERTaV3-small verifiers, and 9B is Qwen3.5-9B; the blue line uses the right axis.
    In (f), MA, GM, RR$^\dagger$, and RN denote MemAgent, GRU-Mem, budget-matched ReMemR1, and native ReMemR1.
    }
    \label{fig:analysis-composite}
    \label{fig:pending-occupancy}
    \label{fig:retention-diagnostics}
    \vspace{-5mm}
\end{figure}

\paragraph{Pending Set Capacity Dynamics.}
The pending set reuses capacity as evidence is resolved.
We track Qwen3.5-4B on a 32-step 2WikiMultiHopQA episode at 1,600 documents with $B_P=256$ (Figure~\ref{fig:analysis-composite}(a)).
Occupancy remains within budget, with three capacity events and a terminal decrease to zero.
Admissions raise occupancy, while promotions and drops release space.
This trajectory illustrates how resolving pending evidence makes capacity available for later sources.

\vspace{-3mm}
\paragraph{Pending Set Capacity Pressure Analysis.}
Accumulating unresolved sources increases pending set pressure.
We test Qwen3.5-4B on 2WikiMultiHopQA at 6,400 documents with $B_P=256$ under ordinary ordering, a 128-chunk gap, and source-dense stress.
Forced drops per admitted pending record rise from 0.5\% to 1.8\% and 7.6\%, respectively (Figure~\ref{fig:analysis-composite}(b)).
These drops reflect competition for capacity before evidence is resolved.
Thus, the fixed budget limits retention when unresolved evidence accumulates (Appendix~\ref{app:retention-diagnostics}).

\vspace{-3mm}
\paragraph{Committed Memory and Pending Set Capacity Allocation.}
The allocation balances pending retention and committed-memory capacity.
We retrain \our with Qwen3.5-4B and varying $B_P$ at $B_P+B_M=1{,}024$ and evaluate HotpotQA at 800 documents.
Answer F1 peaks at 85.5 with $B_P=256$ (Figure~\ref{fig:analysis-composite}(c)).
Smaller allocations limit pending retention; larger ones reduce space for committed facts.
This trade-off supports reserving capacity for both pending evidence and committed memory.

\vspace{-3mm}
\paragraph{Robustness to Delayed Evidence.}
\our remains effective when source relevance emerges later.
We reorder identical passages for 128 HotpotQA questions at 6,400 documents, varying source-to-bridge gaps from 0 to 128 chunks with Qwen3.5-4B.
At 128 chunks, \our reaches 74.5 F1, 25.3 points above the strongest baseline (Figure~\ref{fig:analysis-composite}(d)).
Because available evidence is unchanged, the increasing F1 advantage supports retaining unresolved sources until later context clarifies their relevance.

\vspace{-3mm}
\paragraph{Source-Support Verification Analysis.}
Lightweight source verification balances faithfulness and cost.
We compare verifiers on the same 2,000 blinded promotion attempts, timing each source--fact pair on one H800 at batch size one.
Calibrated DeBERTaV3-small achieves 97.2\% faithfulness at 4.3\,ms/pair, versus Qwen3.5-9B's 97.5\% at 295.0\,ms/pair (Figure~\ref{fig:analysis-composite}(e)).
Calibration reduces false acceptance but increases false rejection.
We adopt the calibrated verifier for similar source support at approximately $69\times$ lower per-pair latency.

\vspace{-3mm}
\paragraph{Computational Efficiency of \our.}
\our trades additional computation for higher answer F1 than the fastest baseline.
We time Qwen3.5-4B on HotpotQA using one H800, complete scans, and matched decoding limits.
At 6,400 documents, \our takes 548\,s/question, 24--42\% less than MemAgent and both ReMemR1 variants (Figure~\ref{fig:analysis-composite}(f)).
It is 15.6\% slower than GRU-Mem but gains 8.8 F1 points (Table~\ref{tab:main}).
These results quantify the cost of improved answer quality.

\subsection{Ablation Analysis}
\label{sec:ablation}

\begin{wraptable}{r}{0.45\linewidth}
\vspace{-7mm}
\centering
\caption{Ablation analysis of \our. Faith. is the percentage of accepted facts supported by their cited sources.}
\label{tab:ablation}
\coemprosefalse
{\fontsize{8}{9.2}\selectfont
\setlength{\tabcolsep}{3.4pt}
\renewcommand{\arraystretch}{1.02}
\begin{tabular}{@{}lrrrr@{}}
\toprule
Variant
& \multicolumn{2}{c}{HotpotQA}
& \multicolumn{2}{c}{2Wiki.} \\
\cmidrule(lr){2-3}\cmidrule(lr){4-5}
& F1 $\uparrow$ & Faith. $\uparrow$
& F1 $\uparrow$ & Faith. $\uparrow$ \\
\midrule
No RL                 & 67.5 & 93.2 & 55.1 & 92.7 \\
Eager, no verifier    & 73.4 & 85.2 & 61.8 & 84.7 \\
Eager + verifier      & 77.2 & 95.7 & 66.0 & 95.2 \\
Pending, no verifier  & 81.5 & 86.7 & 71.3 & 86.2 \\
Outcome-only RL       & 80.4 & 94.4 & 69.9 & 93.9 \\
Fixed-step pending    & 82.3 & 96.2 & 72.2 & 95.7 \\
\midrule
\textbf{CoEM} & \textbf{85.5} & \textbf{97.2} & \textbf{76.1} & \textbf{96.7} \\
\bottomrule
\end{tabular}
\par}
\vspace{-10pt}
\end{wraptable}

Table~\ref{tab:ablation} ablates \our's components using Qwen3.5-4B at 800 documents.
\textbf{No RL} suffers a severe 18.0/21.0 F1 drop, validating the necessity of learned memory decisions.
\textbf{Eager, no verifier} disables pending retention and verification, losing 12.1/14.3 F1 points.
Adding verification (\textbf{Eager + verifier}) recovers 3.8/4.2 F1 points, yet remains 8.3/10.1 points below \our, confirming the advantage of delayed commitment.\textbf{Pending, no verifier} drops 4.0/4.8 F1 points and reduces faithfulness by 10.5 percentage points, highlighting the verifier's role in memory quality. 
Removing step-level evidence rewards (\textbf{Outcome-only RL}) decreases F1 by 5.1/6.2 points. Finally, replacing the learned policy with a rigid two-chunk delay (\textbf{Fixed-step pending}) loses 3.2/3.9 F1 points, favoring adaptive over fixed commitment timing. These results support the complementary contributions of adaptive pending retention, source verification, and step-level evidence rewards, with full \our achieving the best results among the evaluated variants.

\section{Conclusion}
We presented \our{}, a recurrent memory agent for long-context reasoning that addresses delayed evidence relevance, where information that appears unimportant when first observed may become critical only after later context arrives. By separating whether evidence is supported from whether it is relevant, \our{} keeps unresolved excerpts verbatim in a bounded pending set until later context clarifies them, and a frozen NLI verifier admits only source-supported facts. Across three benchmarks and two backbones, \our{} consistently outperforms budget-matched memory agents, with gains that widen as context grows: it leads by 10.4--11.4 F1 points at 6,400 documents with Qwen3.5-9B and by 25.3 points at a 128-chunk evidence gap with Qwen3.5-4B. 
These results suggest that long-context memory should decide not only what to remember, but also when to compress it. Extending this principle to code and scientific documents is a natural next step.

\label{lastmainpage}
\newpage

\section*{Acknowledgments}
This work was supported by the National Research Foundation of Korea (NRF) grant funded by the Korean government (MSIT) under the project ``Development of Risk-Enhanced Continual Learning for Embodied Intelligence in Long-Tail Environments'' (Grant No. RS-2026-25595451) and by the Korea Institute of Science and Technology Information (KISTI) R\&D program through the joint research project ``Development of the Next-Generation Integrated Wired/Wireless Communication Gateway (X-Gateway).''

\section*{AI use statement}
AI tools were used for (i) writing assistance, including manuscript revision, language polishing, and translation; (ii) retrieval and discovery, including literature searches and identification of related work; and (iii) research ideation and execution, including method and experiment design, mathematical analysis, code preparation, and figure generation. The authors reviewed the AI-assisted content, checked cited sources against the original publications, verified mathematical derivations, ran code tests, and cross-checked the reported experimental results and figures. The authors take full responsibility for the final manuscript and accompanying artifacts, including the citations, theoretical claims, implementation, experimental results, and figures.

\section*{Ethics statement}
This study evaluates long-context question answering on public benchmarks (HotpotQA, 2WikiMultiHopQA, and MuSiQue) using published pretrained model components. The system may inherit biases and errors from the datasets, source passages, and models. The frozen verifier assesses whether a proposed fact is supported by the supplied passages; source support does not establish truth beyond those passages or guarantee a correct final answer. Verifier errors and evidence loss under finite memory remain possible.

\section*{Reproducibility statement}
Section~3 specifies the memory controller, token budgets, verification gate, and optimization objective, and Section~4 describes the experimental setup and comparisons. Appendix~A details data construction, training and evaluation settings, prompt templates, verifier calibration and auditing, and additional results. Appendix~B reports computational resources, Appendix~D states the theoretical assumptions and proofs, and Appendix~E specifies controller subroutines and the training objective. The accompanying source code includes the controller, training and evaluation scripts, data-construction utilities, verifier calibration and audit tools, dependency specifications, and unit tests. Code repository: \url{https://github.com/benmagnifico/CoEM}. We will release trained model weights and evaluation manifests upon acceptance.

\bibliography{iclr2027_conference}
\bibliographystyle{iclr2027_conference}
\clearpage
\appendix
\raggedbottom
\section{Experimental details and additional results}
\label{app:experiment-details}
This appendix describes dataset construction, implementation settings, verifier calibration, and additional evaluations of source retention, inference efficiency, and policy training.

\subsection{Dataset construction and evaluation}
\label{app:data-construction}
\paragraph{Questions and length cohorts.}
Policy training uses 32,768 HotpotQA training questions and a disjoint development set of 512 questions.
Following ReMemR1~\citep{shi2025rememr1}, we sample 128 evaluation questions from each official development split with seed 4 and reuse these question IDs across lengths.
The HotpotQA cohort contains hard questions; 2Wiki retains all evidence documents, and MuSiQue contains answerable questions with complete annotated support.
A question is eligible only if its full supporting text fits 4,096 tokens under both policy tokenizers and its question fits $L_q=256$ tokens.
Duplicate IDs and incomplete annotations are excluded before sampling.
The same filtering rule applies to every method, and fixed question IDs allow paired comparisons across lengths.

\paragraph{Document construction.}
For each dataset, we construct a distractor pool from title--paragraph pairs in the training split and remove duplicates using normalized titles and text hashes.
Supporting paragraphs are retained verbatim, including their titles.
For each distractor, we retain the longest complete-sentence prefix containing 96--128 serialized tokens under both tokenizers and skip paragraphs without an eligible prefix.
For question $q$, we exclude distractors with titles matching the supporting documents or text that is nearly identical to the supporting passages.
A stable hash of the question ID and seed 4 determines a permutation without replacement.
We add the first $N_{\rm doc}-|\mathcal S_q|$ distractors, where $\mathcal S_q$ is the support-document set and
$N_{\rm doc}\in\{50,100,200,400,800,1600,3200,6400\}$.
Independent stable keys determine the order of the selected supporting and distractor documents.
Sorting by these keys makes each shorter context a subsequence of every longer context, preserving document order and the exact supporting text.
Training uses the same construction with $N_{\rm doc}=200$ and the training seed.

\paragraph{Serialization and stream boundaries.}
Each document is serialized as \texttt{[DOC id] title: text}, with the question kept outside the document stream.
Consecutive chunks contain at most $C=5{,}000$ tokens without overlap.
The source manifest retains a stable ID for any document that crosses a chunk boundary.
The controller identifies source spans by $\rho=(t,l,r)$, where $[l,r)$ gives token offsets within a chunk.
The manifest maps these coordinates to absolute document offsets for auditing; this mapping does not give the policy access to discarded history.
The source manifest stores the tokenizer revision, total tokens, chunk boundaries, support-document positions, and span offsets.

We control context length by document count, so a given count does not imply a fixed token length.
At 6,400 documents, the distractor prefixes alone contribute approximately 614,000--819,000 tokens; exact lengths depend on the unmodified support passages and headers.
The full-context baseline uses its 262,144-token configured window, including its answer reserve.
We mark an input as unavailable if the complete serialized prompt exceeds this window, without truncating supporting evidence.
The recurrent methods read all chunks under the same context-memory budget.

\paragraph{Answer scoring and uncertainty.}
Before scoring, we lowercase answers, remove punctuation and English articles, and normalize whitespace.
Let $Y_q$ and $\widehat Y_q$ denote the normalized reference and prediction token multisets, and let $c_q$ be their multiset overlap.
For nonempty multisets, $p_q=c_q/|\widehat Y_q|$ and $r_q=c_q/|Y_q|$; an empty or invalid prediction scores zero.
The answer score is
\begin{equation}
 F_1(q)=\begin{cases}2p_qr_q/(p_q+r_q),&p_q+r_q>0,\\0,&\text{otherwise}.\end{cases}
\end{equation}
When several reference answers are acceptable, we use the highest score across references.
If either normalized answer is \texttt{yes}, \texttt{no}, or \texttt{noanswer}, the complete normalized strings must match exactly to receive credit.
We report F1 as 100 times the mean score over questions and record exact match separately.
Eight-length averages give equal weight to each length, and dataset macro-averages give equal weight to each dataset.
The full-context baseline is excluded from all-length averages when lengths are unavailable.
Training seeds are 17, 29, and 43; confidence intervals use 10,000 bootstrap resamples of paired questions.

\subsection{Implementation and matched comparisons}
\label{app:implementation}
\paragraph{Policy and controller configuration.}
Both policies use the official text chat template with reasoning mode disabled and bfloat16 weights and activations.
Each input allows 5,000 chunk tokens, 1,024 context-memory tokens, and 256 question tokens.
The remaining 1,912 tokens cover fixed instructions, the action format, bounded controller feedback, and attempt descriptors, giving an input limit of 8,192 tokens.
Each controller call generates at most $W_\pi=1{,}024$ tokens, and final-answer generation uses $W_{\rm ans}=256$.
The same limits apply to all matched controls.
The serving configuration sets \texttt{max\_model\_len}=9{,}216, covering the full 8,192-token input plus the largest 1,024-token output.
We count fixed instruction tokens at setup and require instructions and temporary feedback to fit within the 1,912-token reserve.
Source text is not silently shortened to accommodate them.
Source identifiers, admission ranks, and separators count toward the budget of the memory component that contains them.

The limit $H=8$ counts all fresh source specifications per chunk, including invalid candidates and immediate promotions.
Each capacity event allows $K=1$ reconsideration pass.
Each participating source record permits at most $J=2$ promotion attempts per chunk. Each proposed fact charges one attempt to every source record it cites, shared across normal calls, capacity handling, and final resolution.
Pending records do not expire with age.
When capacity requires eviction after the bounded reconsideration pass, the controller selects the record with the lowest current admission-order rank.
At the final step, all remaining pending records must be committed or dropped, leaving $P$ empty before answering; we refer to this as final resolution.
The verifier window is $W_V=512$ tokens under its own tokenizer, including the complete premise--claim pair and special tokens.
Table~\ref{tab:hyperparameters} summarizes the policy, controller, training, and evaluation settings.

\begin{table}[t]
\centering
\caption{Policy, controller, training, and evaluation settings.}
\label{tab:hyperparameters}
\small
\begin{tabular}{ll}
\toprule
Setting & Value \\
\midrule
Policy / verifier precision & bfloat16 / float32 logits for calibration \\
Chunk / total context memory & $C=5{,}000$ / $B=1{,}024$ tokens \\
Pending / committed allocations & $B_P=256$ / $B_M=768$ tokens \\
Question / instruction and feedback reserve & $L_q=256$ / 1,912 tokens \\
Input / controller output / final output & 8,192 / $W_\pi=1{,}024$ / $W_{\rm ans}=256$ tokens \\
Total serving window & 9,216 tokens (input plus output) \\
Candidate / reconsideration / attempt caps & $H=8$ / $K=1$ / $J=2$ \\
Verifier window / threshold & $W_V=512$ / $\eta=0.90$ \\
Training questions / documents each & 32,768 / 200 \\
Development / evaluation questions & 512 / 128 per dataset per length \\
Training / evaluation data seeds & 17, 29, 43 / 4 \\
Questions per rollout batch / rollouts & 128 / $G=8$ per question \\
Optimizer updates / PPO epochs per batch & 500 / 1 \\
PPO minibatch / microbatch & 8 trajectories / 1 trajectory per GPU \\
AdamW learning rate / $(\beta_1,\beta_2)$ / $\epsilon_{\rm opt}$ & $10^{-6}$ / $(0.9,0.999)$ / $10^{-8}$ \\
Weight decay / gradient norm cap & 0.01 / 1.0 \\
Warmup / schedule after warmup & 20 updates / constant \\
PPO clip / reference KL coefficient & $\epsilon=0.2$ / $\beta=10^{-3}$ \\
Outcome weight / rejection / invalid penalties & $\alpha=0.8$ / $\lambda_v=0.20$ / $\lambda_f=0.05$ \\
Return discount / advantage standardization & $\gamma=1$ / none (group centering only) \\
Rollout temperature / top-$p$ / top-$k$ & 1.0 / 1.0 / disabled \\
Evaluation decoding / repetition penalty & greedy / 1.0 \\
Checkpoint interval / selection metric & 25 updates / development answer F1 \\
Evaluation bootstrap / resamples & paired questions / 10,000 \\
\bottomrule
\end{tabular}
\end{table}

\paragraph{Training and serving configuration.}
All trainable comparisons use the same policy initialization, question batches, and 500-update limit.
Training uses full-parameter FSDP, gradient checkpointing, and gradient accumulation to reach the stated minibatch size.
We use a frozen reference policy and one PPO epoch per sampled batch.
Truncated generations that fail to complete the required action format count as invalid actions and receive the format penalty.
Their generated tokens remain in the policy objective.
Policy updates use the normalized generated-token objective in Section~\ref{sec:method}, with no critic and no reward bonus for accepted writes.
Checkpoint selection evaluates the same 512 development questions every 25 updates and chooses the earlier checkpoint when scores are tied.

The implementation uses PyTorch, Transformers, a recurrent extension of the verl trainer, and vLLM serving.
Both Qwen3.5-4B and Qwen3.5-9B are trained on one node with eight NVIDIA H800 GPUs. Evaluation uses NVIDIA H800 and A100 GPUs, with a single GPU type within each synchronous group.
Serving uses tensor-parallel sizes of 1 for Qwen3.5-4B and 2 for Qwen3.5-9B.
Maximum GPU-memory utilization is 0.80, and the scheduler token budget is 32,768.
The full-context baseline uses tensor parallelism up to 8 to fit the complete prompt.
We report its latency separately from the single-GPU recurrent timing study.

\paragraph{Memory-agent baselines.}
MemAgent receives all 1,024 retained tokens as textual context memory.
GRU-Mem retains its learned update action, with early exit disabled to separate memory selection from the amount of evidence observed.
The budget-matched ReMemR1 adaptation allocates 480 tokens to current memory, 480 to one complete historical snapshot, and 64 to the callback query.
These allocations total 1,024 serialized tokens.
Replacing the archived snapshot discards the older snapshot permanently.

Native ReMemR1 keeps all historical memory snapshots outside this allowance and uses a 1,024-token current memory and a 64-token query.
We record archive bytes and recalled prompt tokens separately.
Callbacks access retained summaries but cannot recover discarded source text that those summaries omit.
All comparisons use Qwen3.5-4B and Qwen3.5-9B.
Table~\ref{tab:native-archive} reports this comparison separately because the methods do not share the same storage budget.

\begin{table}[t]
\centering
\caption{Length-averaged F1 (\%). Native ReMemR1 retains historical snapshots outside the 1,024-token budget and is reported separately from the matched-storage comparison.}
\label{tab:native-archive}
\small
\setlength{\tabcolsep}{4pt}
\begin{tabular}{llrrr}
\toprule
Model & Dataset & ReMemR1$^\dagger$ & ReMemR1 (native) & CoEM \\
\midrule
Qwen3.5-4B & HotpotQA & 81.2 & 83.4 & 85.7 \\
Qwen3.5-4B & 2Wiki & 70.9 & 73.0 & 76.4 \\
Qwen3.5-4B & MuSiQue & 61.4 & 63.3 & 67.0 \\
Qwen3.5-9B & HotpotQA & 84.2 & 86.4 & 89.6 \\
Qwen3.5-9B & 2Wiki & 75.6 & 78.0 & 81.9 \\
Qwen3.5-9B & MuSiQue & 67.2 & 69.3 & 73.6 \\
\bottomrule
\end{tabular}
\end{table}

\paragraph{Ablation definitions.}
Removing pending retention sets $B_P=0$ and $B_M=1{,}024$, so source detail must be committed or discarded in its arrival chunk.
Removing verification bypasses only the entailment check and sets $\lambda_v=0$; source matching, operation validity, and memory budgets remain enforced.
We retrain each of the four pending-retention and verification combinations independently under matched settings.
Outcome-only RL sets $\alpha=1$ and both local penalties to zero.
The no-RL variant uses the initialized policy with the same inference interface.
The fixed-step pending variant attempts promotion exactly two chunks after admission, or during final resolution if it occurs earlier.
It drops rejected or capacity-infeasible records.
These controls separately test source retention, support verification, and the learned timing of commitment.

\subsection{Prompt templates and action format}
\label{app:prompt-templates}

\begingroup
We use the following prompt templates for reading, capacity handling, final resolution, and final-answer generation.
Braced uppercase fields are filled by the controller before a call; they are never instructions to fetch historical text.
The controller supplies the question, current chunk, and serialized $M$ and $P$ using the backbone's chat template.
Call-level instructions, limits, and bounded feedback are included in the 1,912-token instruction-and-feedback reserve.
Committed-entry IDs, field labels, and delimiters within the displayed committed memory count toward $B_M$; pending-source metadata, admission ranks, and their field labels and delimiters count toward $B_P$.
If a filled instruction block exceeds this reserve, the configuration fails validation before evaluation; source text is not truncated to make it fit.

\paragraph{Reading-policy system prompt.}
\begin{quote}\small
You read a document stream to answer one question. You may use only the
question, the current chunk, committed memory, and pending excerpts shown
in this call. Source coordinates are provenance, not retrieval commands.
Do not infer that discarded text remains available.

Return one JSON object matching the action format below, with no prose
or Markdown. Select at most H fresh source spans from the current chunk.
Use the displayed tokenizer coordinates [chunk, start, end), with an
exclusive end offset. Never generate replacement source quotations.

For each selected or pending excerpt, choose Promote, Keep, or Drop.
Keep preserves the exact excerpt. Drop removes it. Promote proposes
compact facts supported by available excerpts. Write each independently
checkable relation in its own fact field and cite all spans needed to
support that relation. Committed facts may guide relevance decisions,
but are not source evidence for a new or rewritten fact.

A promotion may remove whole committed entries and insert new facts in
one atomic transaction. Every inserted fact must pass verification and
the resulting serialized memory must fit its budget. Do not assume a
proposed transaction has succeeded. On rejection, the existing memory
and supporting pending excerpts are preserved, subject to forced removal.
Each fact consumes one promotion attempt from every source record it
uses; obey the remaining attempt counts supplied in this call.

Successful promotion of a pending target must consume that target.
If several facts need the same excerpt, include them in one transaction
before releasing it. Preserve unrelated pending sources unless you
explicitly consume or drop them. Do not include future plans, hidden
scratchpads, or unsupported answers in either memory store.

Obey MODE and the supplied target restrictions. During forced resolution
or final resolution, Keep is unavailable for the designated record.
\end{quote}

\paragraph{Reading-policy user prompt.}
\begin{quote}\small\ttfamily
MODE: \{MODE\}\\
QUESTION: \{QUESTION\}\\
CHUNK INDEX: \{CHUNK\}\quad FINAL CHUNK: \{IS\_FINAL\}\\
BUDGETS: M=\{B\_M\}, P=\{B\_P\}, total=\{B\}\quad
LIMITS: H=\{H\}, K=\{K\}, J=\{J\}\\
CURRENT CHUNK WITH TOKEN COORDINATES:\\
\{CURRENT\_CHUNK\}\par
COMMITTED MEMORY (call-local IDs, fact, provenance):\\
\{COMMITTED\_MEMORY\}\par
PENDING RECORDS (source span, verbatim text, admission rank):\\
\{PENDING\_RECORDS\}\par
REMAINING PROMOTION ATTEMPTS: \{ATTEMPT\_COUNTS\}\\
FORCED TARGET, IF ANY: \{FORCED\_TARGET\}\\
CURRENT CAPACITY REQUIREMENT: \{REQUIRED\_SPACE\}\\
FEEDBACK FROM THIS CHUNK: \{BOUNDED\_FEEDBACK\}\\
Return the next JSON action object.
\end{quote}

\paragraph{Mode-specific restrictions.}
In \texttt{read} mode, fresh candidates and actions on pending records are allowed.
In \texttt{reconsider} mode, the controller requests another decision on the existing pending records to accommodate an already validated candidate.
In \texttt{forced} mode, \texttt{pending\_actions} contains exactly one Promote or Drop operation for the designated oldest record.
In \texttt{terminal} mode, the same restriction applies to the next record in admission order until $P$ is empty.
The \texttt{candidates} array must be empty in \texttt{reconsider}, \texttt{forced}, and \texttt{terminal} modes; these calls cannot admit additional sources.
The controller, not the model, determines whether another call is permitted under $K$ and $J$.
An invalid forced response uses the fallback specified in Appendix~\ref{sec:algorithm-details}.

\paragraph{Action format.}
The following typed grammar defines the complete response structure.
Square-bracketed type names denote arrays; literal strings are case sensitive.
\texttt{Span} always contains three integers, and \texttt{Transaction} is required exactly when \texttt{action} is \texttt{"Promote"}.
Keep and Drop objects have no transaction field.
Unknown fields, repeated fresh spans, out-of-range offsets, unavailable sources, and references to missing committed entries are rejected as invalid operations.
\begin{quote}\footnotesize
\begin{minipage}{\linewidth}
\begin{verbatim}
Span        := [chunk_index, token_start, token_end]
Fact        := {"fact": string, "sources": [Span, ...]}
Transaction := {"remove": [memory_id, ...],
                "insert": [Fact, ...],
                "consume": [Span, ...]}
Candidate   := {"span": Span,
                "action": "Promote" | "Keep" | "Drop",
                "transaction": Transaction}
PendingAction := {"target": Span,
                  "action": "Promote" | "Keep" | "Drop",
                  "transaction": Transaction}
Response    := {"candidates": [Candidate, ...],
                "pending_actions": [PendingAction, ...]}
\end{verbatim}
\end{minipage}
\end{quote}
The arrays \texttt{remove} and \texttt{consume} may be empty; \texttt{insert} and each fact's \texttt{sources} are nonempty.
The full target span identifies a pending record; a cited subspan is legal only when its complete text remains available within a retained record or the current chunk.
The \texttt{consume} list names complete pending records, and includes the target of a successful pending promotion.
At least one inserted fact must cite the target's retained source; a transaction cannot count as promotion of a target that supports none of its new facts.
Committed-memory IDs are assigned to the entries displayed at the start of a call.
They identify those entries throughout that call, even if earlier actions remove other entries; they are regenerated at the next call and do not carry additional historical information.
Serialized committed memory remains a list of $(f,\rho)$ entries.

The controller validates fresh spans before executing any action.
It then executes fresh immediate promotions in candidate order and operations on existing pending records in \texttt{pending\_actions} order, followed by admission of fresh Keep candidates in candidate order.
A fresh Drop candidate is discarded after validation.
Each operation is checked against the state left by preceding operations, so a source consumed earlier in the call cannot support a later promotion.
A failed fresh promotion does not itself create a pending record; any later use of its source remains subject to current-chunk availability and the same attempt budget.
The controller enforces $H$ across all fresh specifications for the chunk and $J$ across all participating source records and all calls in that chunk.

\paragraph{Two-fact transaction example.}
The example record $[3,20,50)$ contains the source excerpt
\emph{``Rina was born in Harbor City. Rina founded Delta Lab.''}
Source coordinates specify the chunk and token offsets stored in the source manifest.
With two promotion attempts remaining for this record, the response proposes both facts in one atomic transaction:
\begin{quote}\footnotesize
\begin{minipage}{\linewidth}
\begin{verbatim}
{
  "candidates": [],
  "pending_actions": [{
    "target": [3, 20, 50],
    "action": "Promote",
    "transaction": {
      "remove": [],
      "insert": [
        {"fact": "Rina was born in Harbor City.",
         "sources": [[3, 20, 50]]},
        {"fact": "Rina founded Delta Lab.",
         "sources": [[3, 20, 50]]}
      ],
      "consume": [[3, 20, 50]]
    }
  }]
}
\end{verbatim}
\end{minipage}
\end{quote}
The controller first validates both source--fact pairs, checks both attempt charges and the proposed memory size, and obtains both verifier decisions.
Only if all checks pass does it insert both entries and release the pending record.
If either claim is rejected, neither insertion nor removal is applied; during ordinary reading the excerpt remains available.
This ordering prevents the first accepted fact from releasing the only source needed by the second.

\paragraph{Verifier inputs.}
For the DeBERTaV3-small verifier, the controller uses the checkpoint tokenizer's paired-input interface with $S_e$ as the premise and $f$ as the hypothesis, including its special tokens.
It concatenates multiple cited excerpts in citation order with source-boundary separators, checks the complete pair against $W_V$, and computes the calibrated entailment probability.
No natural-language instruction is prepended to this classifier input.
For the generative-verifier comparison, the separate prompt is:
\begin{quote}\small
Decide whether the PREMISE supports the complete CLAIM. Use only the
premise. Return exactly one label: entailment, contradiction, or neutral.
Choose neutral if the premise does not establish every part of the claim.
Do not answer the original question and do not use outside knowledge.\par
PREMISE: \{COMPLETE\_SOURCE\_PREMISE\}\par
CLAIM: \{PROPOSED\_FACT\}
\end{quote}

\paragraph{Final-answer template.}
After final resolution, the answer call receives the following system message and the question plus committed memory only:
\begin{quote}\small
Answer the question using the committed facts supplied below. Return the
shortest answer phrase that fully answers the question, without reasoning,
citations, JSON, or additional commentary. If the committed facts do not
support an answer, return exactly ``unknown''. Do not reconstruct source
text from provenance identifiers.\par
QUESTION: \{QUESTION\}\par
COMMITTED FACTS: \{FINAL\_COMMITTED\_MEMORY\}
\end{quote}
The answer uses greedy decoding and at most $W_{\rm ans}=256$ generated tokens.
The current chunk, pending records, earlier action text, and verifier feedback are not included in this final call.
\endgroup

\subsection{Verifier calibration and independent audit}
\label{app:verifier-details}
\label{sec:verifier-details}
\paragraph{Checkpoint and evidence pairs.}
The verifier uses a DeBERTaV3-small NLI model~\citep{he2023debertav3} from the public \texttt{cross-encoder/}\allowbreak\texttt{nli-deberta-v3-small} checkpoint~\citep{crossencoder2024nlideberta}.
This checkpoint was trained on SNLI~\citep{bowman2015snli} and MultiNLI~\citep{williams2018mnli}.
The logits follow the order contradiction, entailment, and neutral, and the classifier remains frozen throughout policy training and inference.
Each premise concatenates the complete source spans cited by one proposed fact, with source IDs as separators; the fact is the claim assessed by NLI.
The spans must still be visible in the current chunk or pending set.
For a claim citing several sources, the verifier checks the joint premise rather than source identifiers or compressed memory alone.
Pairs that exceed the verifier window are rejected without truncating evidence.
Within the same attempt cap, the policy may then propose a shorter claim that can be verified separately.

\paragraph{Calibration protocol.}
We use three question-disjoint sets of 1,000 source--claim pairs each for temperature fitting, threshold selection, and held-out calibration evaluation.
These sets are disjoint from policy training, checkpoint selection, final evaluation, and audit questions.
Pairs include faithful compression, entity substitution, reversed relations, and unsupported additions.
Two annotators independently label each pair and adjudicate disagreements to provide the source-support reference.
For frozen logits $z_\phi(S_e,f)$ and NLI labels $y_i$, we fit one positive temperature:
\begin{equation}
\widehat T_{\rm cal}=\arg\min_{T_{\rm cal}>0}\sum_{i\in\mathcal C_{\rm fit}}
-\log\!\left[\operatorname{softmax}\!\left(z_\phi(S_{e_i},f_i)/T_{\rm cal}\right)\right]_{y_i}.
\end{equation}
The initialization is $T_{\rm cal}=1$; we optimize its logarithm using L-BFGS for at most 50 iterations with tolerance $10^{-6}$.
We select $\eta$ from $\{0.70,0.75,0.80,0.85,0.90,0.95\}$ to maximize supported-claim recall while keeping false acceptance at or below 5\% on the threshold-selection split.
Ties are resolved in favor of the larger threshold.
The selected threshold is $\eta=0.90$.
Table~\ref{tab:calibration} reports ten-bin ECE, three-class negative log likelihood, and binary entailment Brier score.
Temperature scaling follows~\citet{guo2017calibration}.

\begin{table}[t]
\centering
\caption{Held-out calibration metrics. ECE uses ten equal-width bins; lower values are better.}
\label{tab:calibration}
\small
\setlength{\tabcolsep}{4pt}
\begin{tabular}{lrrr}
\toprule
Configuration & ECE (\%) $\downarrow$ & NLL $\downarrow$ & Brier ($\times100$) $\downarrow$ \\
\midrule
Raw & 7.10 & 0.52 & 8.00 \\
Temperature scaling & 3.20 & 0.49 & 7.40 \\
\bottomrule
\end{tabular}
\end{table}

\paragraph{Blinded audit and metric definitions.}
The audit uses 2,000 attempted promotions from disjoint question IDs, collected before verification.
Every verifier therefore receives the same set of supported and unsupported proposals.
Two human annotators independently assess whether the displayed premise supports each claim and adjudicate disagreements to establish reference labels.
A separate \texttt{FacebookAI/roberta-large-mnli} classifier~\citep{liu2019roberta,facebook2026robertamnli} serves as an auxiliary judge rather than the reference.
For true positives TP, false positives FP, true negatives TN, and false negatives FN,
\begin{equation}
\mathrm{Faithfulness}=\frac{\mathrm{TP}}{\mathrm{TP}+\mathrm{FP}},\qquad
\mathrm{FAR}=\frac{\mathrm{FP}}{\mathrm{FP}+\mathrm{TN}},\qquad
\mathrm{FRR}=\frac{\mathrm{FN}}{\mathrm{TP}+\mathrm{FN}}.
\end{equation}
An accepted claim is a positive prediction, and an entailed claim is a positive reference.
Metrics with zero denominators are undefined.
The ablation faithfulness metric uses a separate blinded sample of accepted entries from each variant.

\begin{table}[t]
\centering
\caption{
Independent audit of source-support verifiers on the same blinded set of 2,000 pre-verification promotion attempts. Faithfulness measures source support among accepted claims; FAR and FRR measure unsupported claims accepted and supported claims rejected, respectively. Latency is measured per source--claim pair on one H800 with batch size one.
}
\label{tab:verifier-audit}
\vspace{4pt}
\begin{tabular}{lrrrr}
\toprule
Verifier & Faith. $\uparrow$ & FAR $\downarrow$ & FRR $\downarrow$ & ms/pair $\downarrow$ \\
\midrule
No verification                & 64.0 & 100.0 & 0.0 &   0.0 \\
Policy self-check (Qwen3.5-4B) & 89.8 &  18.5 & 8.3 & 165.0 \\
DeBERTaV3-small, raw             & 94.9 &   9.2 & 3.9 &   4.2 \\
DeBERTaV3-small, calibrated      & 97.2 &   4.9 & 6.4 &   4.3 \\
Qwen3.5-9B verifier            & 97.5 &   4.3 & 5.8 & 295.0 \\
\bottomrule
\end{tabular}
\end{table}

\paragraph{Small and large verifier comparison.}
All verifiers receive the same complete premise--claim pairs.
The policy self-check and Qwen3.5-9B generative verifiers each emit one NLI label using greedy decoding and an eight-token output cap.
Their calibration uses disjoint fitting data.
We measure per-pair latency on one H800 with batch size one, ten warm-up pairs, and device synchronization, and report resident memory separately.

The small verifier takes 4.3\,ms/pair, compared with 295\,ms/pair for Qwen3.5-9B, with faithfulness of 97.2\% and 97.5\%, respectively.
The small verifier therefore achieves similar source faithfulness on these short pairs at substantially lower latency.
Calibration reduces false acceptance from 9.2\% to 4.9\% while increasing false rejection from 3.9\% to 6.4\%.

\subsection{Evidence order, context-memory occupancy, and capacity pressure}
\label{app:retention-diagnostics}
\label{sec:diagnostics-extra}
\paragraph{Budget and source-to-bridge gap interventions.}
The budget sweep varies $B_P\in\{0,\allowbreak 128,\allowbreak 256,\allowbreak 384,\allowbreak 512,\allowbreak 768\}$ with $B_M=1{,}024-B_P$ and matched retraining.
The zero- and 256-token rows coincide with their verifier-enabled ablations.

The source-to-bridge gap probe tests how long unresolved evidence must be retained.
Annotated supporting relations identify an intermediate entity and a later bridge that connects it to the question.
We select 128 HotpotQA questions after checking that each 6,400-document stream contains at least 130 chunks under both tokenizers.
All methods use the same question and document IDs at gaps of 0, 2, 8, 32, 64, or 128 chunks.
Only passage order changes across gap settings.
Infeasible candidates are replaced from the remaining eligible pool before model evaluation.
The intervention thus varies the delay before the policy can connect the retained evidence to the question.

\paragraph{Natural-order frequency.}
We identify delayed evidence when a source relation precedes the bridge linking its intermediate entity to the question, with at least one intervening chunk boundary.
Two annotators trace the annotated evidence dependencies in randomly selected cohorts from HotpotQA, 2WikiMultiHopQA, and MuSiQue.
Each cohort contains 128 questions with 800 documents per question.
Across these benchmark samples, 32.03\%--44.53\% of questions contain evidence whose relevance becomes apparent only later.
These frequencies measure how often delayed relevance occurs under natural document order, complementing the controlled gap intervention.

\paragraph{Supporting-relation coverage.}
For the 6,400-document cohort, we record committed and pending token occupancy and recoverable supporting relations over the first 32 nonterminal transitions.
For already-seen annotated supporting relations $\mathcal E_t$, the coverage metric is
\begin{equation}
\mathrm{Coverage}_t=\frac{\sum_{e\in\mathcal E_t}\mathbf 1\{e\text{ is recoverable from the retained context at }t\}}
{|\mathcal E_t|},\qquad |\mathcal E_t|>0.
\end{equation}
We omit steps before any supporting relation appears rather than inserting a pseudo-count into the denominator.
Blinded annotators assess whether each relation can be recovered from the text retained by each method, including pending source excerpts.

Figure~\ref{fig:memory-utilization} presents the Qwen3.5-4B HotpotQA experiment with 6,400 documents and 128 reading chunks, using the same occupancy trace as Figure~\ref{fig:memory-occupancy} in the main text.
The curves describe retained textual memory during reading; the pending-set component is included in \our's total context-memory budget.
The final plotted reading-state sample precedes the final resolution of pending evidence.
The final answer is generated only after the remaining pending records have been resolved, leaving $P_T=\varnothing$.

\begin{figure}[t]
\centering
\includegraphics[width=0.7\linewidth]{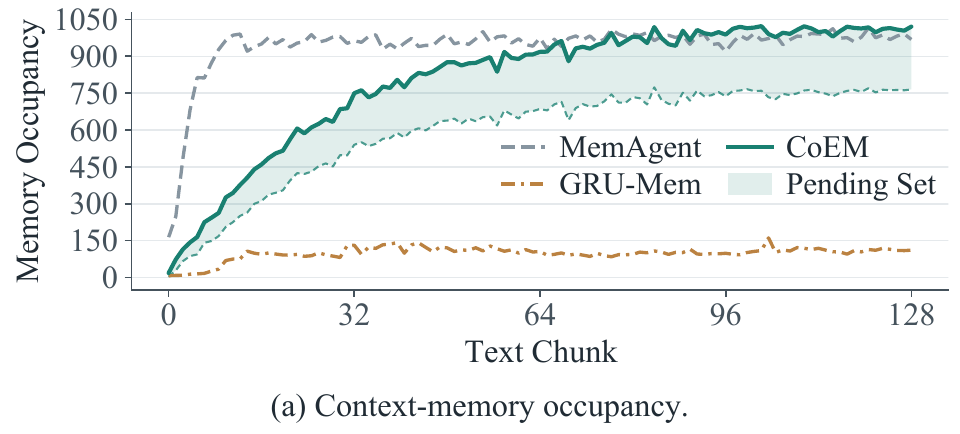}
\caption{Context-memory occupancy during the Qwen3.5-4B HotpotQA experiment with 6,400 documents and 128 chunks. The shaded band shows pending source text above committed memory; both components contribute to \our's total retained context memory. This enlarged view uses the same experimental trace as Figure~\ref{fig:memory-occupancy}.}
\label{fig:memory-utilization}
\end{figure}

\paragraph{Pending-set occupancy and capacity events.}
Figure~\ref{fig:analysis-composite}(a) shows an example trajectory in which pending occupancy rises as unresolved sources arrive and falls as they are promoted or dropped.
Additional reconsideration under capacity pressure occurs only when the next eligible source does not fit; the controller does not evict sources at every step.
The trace includes three such events and final resolution, with $\ell(P_t)\leq256$ throughout.

A separate 2Wiki study tests ordinary, long-gap, and source-dense document orderings (Table~\ref{tab:pending-pressure}).
We report mean and 95th-percentile pending occupancy, admission-triggered capacity events, forced eviction of pending records, and token-budget violations.
The capacity-event rate uses eligible source admissions as its denominator.
The forced-drop rate uses admitted pending records, with each evicted record counted at most once.
No token-budget violations occur in these conditions.
The stress condition tests a limitation of bounded memory: verification cannot recover an informative excerpt once it has been discarded.

\begin{table}[t]
\centering
\caption{Pending-set stress diagnostics for Qwen3.5-4B on 2WikiMultiHopQA with 6,400 documents and $B_P=256$. Mean / P95 tokens report the mean and 95th-percentile occupancy of $P$. The capacity-event rate is the fraction of eligible admissions that require freeing space; the forced-drop rate counts evicted resident pending records per admitted record; overflow denotes token-budget violations.}
\label{tab:pending-pressure}
\vspace{4pt}
{\renewcommand{\arraystretch}{1.15}
\begin{tabular}{lrrrr}
\toprule
2Wiki condition & Mean / P95 tokens & Capacity events & Forced drops & Overflow \\
\midrule
Ordinary ordering   & 142 / 223 &  2.8\% & 0.5\% & 0.0\% \\
Gap 128             & 183 / 246 &  6.2\% & 1.8\% & 0.0\% \\
Source-dense stress & 224 / 254 & 19.7\% & 7.6\% & 0.0\% \\
\bottomrule
\end{tabular}}
\end{table}

\subsection{End-to-end efficiency}
\label{app:efficiency}
The timing study uses one H800 per recurrent inference stream, batch size one, bfloat16 weights, and identical decoding limits.
Cross-request prefix caching is disabled.
Both policy scales use tensor-parallel size one in a dedicated timing deployment to hold hardware constant.
The timer runs from before the first recurrent prompt until completion of the final answer.
It includes serialization, all policy calls, rejected proposals, verifier calls, capacity handling, final resolution, and any callback or archive operations.
After ten warm-up questions, we time the same 128 questions for all methods and report mean latency together with token and call counts.
Asynchronous component timers are diagnostic and are not summed when operations overlap.
Peak allocated device memory, verifier-resident memory, archive bytes, processed tokens, generated tokens, and callback tokens are logged separately.
Early stopping remains disabled.

\begin{table}[t]
\centering
\caption{End-to-end latency (s/question) for the Qwen3.5-4B policies on one H800. Full-scan timing includes all controller operations, verification, memory callback, and final-answer generation.}
\label{tab:efficiency}
\small
\setlength{\tabcolsep}{4pt}
\begin{tabular}{lrrrr}
\toprule
Method & 800 docs & 1,600 docs & 3,200 docs & 6,400 docs \\
\midrule
MemAgent & 84 & 171 & 351 & 720 \\
GRU-Mem & 55 & 112 & 231 & 474 \\
ReMemR1$^\dagger$ & 102 & 207 & 425 & 870 \\
ReMemR1 (native) & 108 & 221 & 459 & 951 \\
\textbf{CoEM} & 63 & 130 & 267 & 548 \\
\bottomrule
\end{tabular}
\end{table}

Table~\ref{tab:efficiency} supplies the numerical values plotted in Figure~\ref{fig:analysis-composite}(f).
The verifier comparison is shown in Figure~\ref{fig:analysis-composite}(e), with its calibration and audit protocol in Appendix~\ref{app:verifier-details}.

\subsection{Training curves and memory decisions}
\label{app:training-diagnostics}

\begingroup
We examine how answer quality and memory decisions change during policy optimization.
The comparison uses Qwen3.5-4B with the 200-document training construction, 32,768 training questions, 128 questions per rollout batch, eight rollouts per question, and 500 training updates.
Both methods use the same 1,024-token context-memory budget, 256-token pending allocation, 5,000-token chunks, and frozen verifier at $\eta=0.90$.
\our combines answer and evidence advantages with $\alpha=0.8$; Outcome-only RL retains the same inference controller and verifier but trains with $\alpha=1$ and no local penalties.
Three runs use seeds 17, 29, and 43.

\paragraph{Answer quality and checkpoint selection.}
Figure~\ref{fig:rl-performance} reports answer F1 on training rollouts at every update and on the fixed 512-question development cohort at 200 documents every 25 updates.
The mean training F1 at update 500 is 87.3 for \our and 82.7 for Outcome-only RL.
The highest mean development F1 is 86.7 at update 450 for \our and 81.9 at update 425 for Outcome-only RL; the final-update values are 86.4 and 81.7.
Thus, the best development checkpoint need not be the last update.
Each run selects its own checkpoint using the development protocol in Appendix~\ref{app:implementation}; the marked maxima of the mean curves summarize the plotted runs and are not a replacement for per-run checkpoint selection.
Shaded bands show one sample standard deviation across policy seeds.

\paragraph{Verifier rejections and invalid operations.}
At update 500, the verifier-rejection rate is 5.2\% for \our and 13.6\% for Outcome-only RL, measured among proposals that pass the non-entailment checks.
Invalid-operation rates are 0.8\% and 1.3\%, respectively, using all proposal attempts as their denominator.
The two denominators are kept separate because an invalid operation is not also counted as a verifier rejection.
These quantities distinguish source-verification failures from invalid operations.
\endgroup

\begin{figure}[t]
\centering
\includegraphics[width=\linewidth]{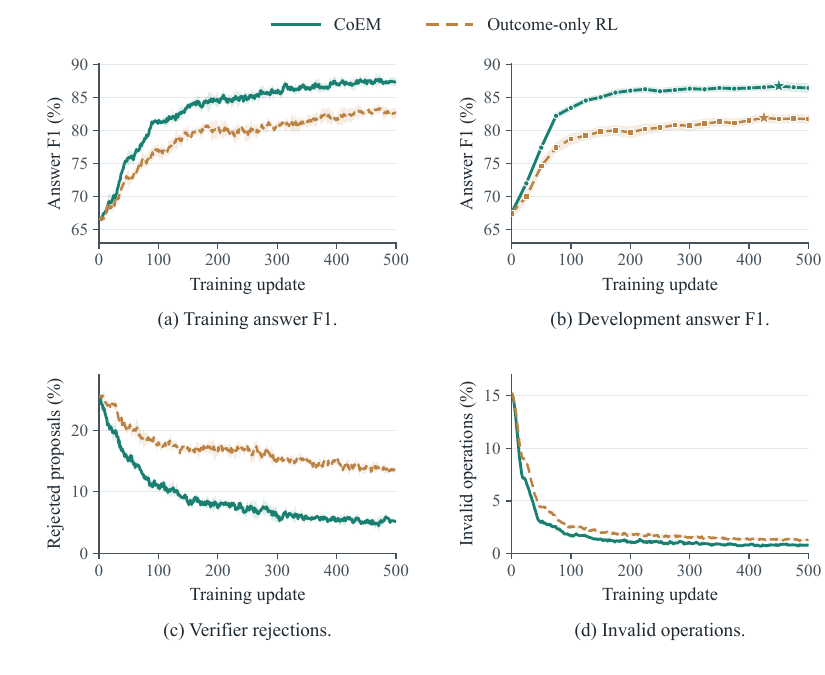}
\caption{Training diagnostics for Qwen3.5-4B. (a) Training answer F1. (b) Development answer F1, evaluated every 25 updates; stars mark the maxima of the mean curves, rather than the checkpoints selected separately for each run. (c) Verifier-rejection rate among proposals passing all non-entailment checks. (d) Invalid-operation rate among all proposal attempts. Lines show three-seed means and shaded bands show one sample standard deviation across policy seeds.}
\label{fig:rl-performance}
\end{figure}

\begingroup
\paragraph{Memory decisions and step-level evidence rewards.}
Figure~\ref{fig:rl-behavior} reports step-level evidence rewards, fact proposals checked by the verifier, proposals passing verification, and voluntary \textsc{Drop} actions per reading chunk.
For both methods, the displayed evidence reward is nonpositive and is computed using $\lambda_v=0.20$ and $\lambda_f=0.05$; it is a diagnostic only for Outcome-only RL and does not enter that method's objective.
At update 500, \our makes 2.55 eligible proposals per chunk, of which approximately 2.42 pass the verifier, compared with 2.15 and 1.86 for Outcome-only RL.
The corresponding mean evidence rewards are $-0.0275$ and $-0.0599$ per chunk.
More proposals pass verification while fewer are rejected, so the smaller penalty is not explained simply by avoiding proposals.
Verification evaluates each proposed fact: an atomic transaction containing several facts still fails if any other fact or capacity check fails.
Voluntary \textsc{Drop} rates are 0.24 and 0.42 per chunk; forced drops under capacity pressure are excluded from this counter and are reported separately in Table~\ref{tab:pending-pressure}.
\endgroup

\begin{figure}[t]
\centering
\includegraphics[width=\linewidth]{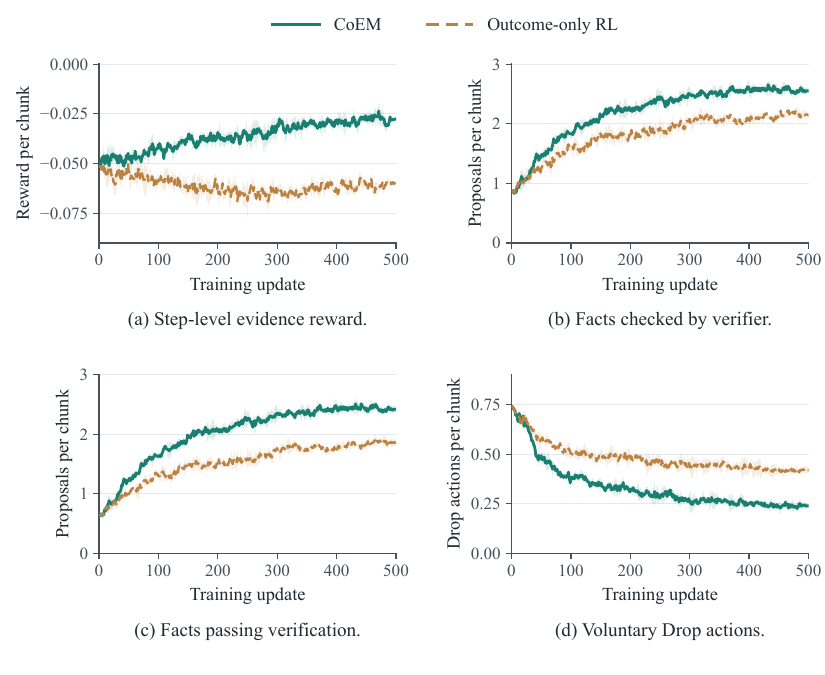}
\caption{Memory decisions during training. (a) Step-level evidence reward. (b) Fact proposals checked by the verifier. (c) Proposals passing verification. (d) Voluntary \textsc{Drop} actions. Panels (b)--(d) report counts per reading chunk. Lines show three-seed means and shaded bands show one sample standard deviation across policy seeds. Passing verification does not imply that the corresponding atomic transaction is committed. Evidence rewards for Outcome-only RL are logged but do not affect its updates.}
\label{fig:rl-behavior}
\end{figure}

\subsection{Variation across policy seeds}
\label{app:seed-uncertainty}

\begingroup
Table~\ref{tab:seed-uncertainty} separates training variation from uncertainty over evaluation questions for the central comparisons.
Each setting reports the mean and sample standard deviation over seeds 17, 29, and 43.
For a paired difference, we resample the 128 question IDs 10,000 times, using the same resampled IDs for both methods and all three policy seeds, and recompute the mean difference.
The 2.5th and 97.5th percentiles give the displayed interval.
These question-bootstrap intervals condition on the available seeds; they do not replace the reported seed standard deviations.
Each row concerns one dataset and one context length, so the table does not treat repeated lengths as additional independent questions.
\endgroup

\begin{table}[!htbp]
\centering
\caption{Uncertainty summaries for key comparisons. Settings list backbone / dataset / number of documents. F1 entries are mean $\pm$ SD over policy seeds 17, 29, and 43; $\Delta$ is CoEM minus comparator F1. Each setting contains 128 paired questions, and the 95\% interval for $\Delta$ uses 10,000 question-level bootstrap resamples that preserve method and seed pairing. $\dagger$ denotes budget-matched ReMemR1.}
\label{tab:seed-uncertainty}
\vspace{4pt}
{\renewcommand{\arraystretch}{1.15}
\begin{tabular}{llrrr}
\toprule
Setting & Comparator & CoEM & Comparator & $\Delta$ [95\% CI] \\
\midrule
4B / HotpotQA / 800   & Outcome-only RL   & $85.5\pm0.46$ & $80.4\pm0.85$ & 5.1 [3.3, 7.0] \\
4B / 2Wiki / 800      & Outcome-only RL   & $76.1\pm0.56$ & $69.9\pm0.85$ & 6.2 [4.3, 8.2] \\
\addlinespace[3pt]
9B / HotpotQA / 6,400 & ReMemR1$^\dagger$ & $89.1\pm0.40$ & $78.7\pm0.75$ & 10.4 [7.6, 13.4] \\
9B / 2Wiki / 6,400    & ReMemR1$^\dagger$ & $81.0\pm0.56$ & $69.6\pm0.75$ & 11.4 [8.7, 14.3] \\
9B / MuSiQue / 6,400  & ReMemR1$^\dagger$ & $72.3\pm0.66$ & $61.6\pm0.85$ & 10.7 [8.2, 13.4] \\
\bottomrule
\end{tabular}}
\end{table}

\subsection{Supporting-evidence retention and commitment timing}
\label{app:gold-retention}

\begingroup
The occupancy measurements describe how much text is retained; here we examine whether the retained text preserves annotated supporting relations.
The comparison uses Qwen3.5-4B on 128 HotpotQA questions with 6,400 documents, with the same policy seeds and memory allocations as above.
Coverage uses the definition in Appendix~\ref{app:retention-diagnostics} and includes both committed facts and pending source excerpts.
We sample 32 equally spaced fractions of each document stream, with the final sample taken before final resolution, rather than restricting measurement to the first 32 chunks.
At a checkpoint, a question contributes to the mean only after at least one annotated supporting relation has been observed.
This separates evidence recoverability from total token occupancy.

Figure~\ref{fig:evidence-retention}(a) shows that coverage immediately before final resolution is 86.2\% for \our and 74.8\% for Outcome-only RL, a difference of 11.4 percentage points.
For the source-to-bridge gap intervention, panel (b) measures the fraction of initially admitted pending records containing annotated supporting evidence that remain in $P$ immediately before their bridge evidence arrives.
At a 128-chunk gap, this source-survival rate is 89.8\% for \our and 73.4\% for Outcome-only RL.
An excerpt promoted before the bridge is absent from $P$ even if its supporting relation survives in $M$, so source survival and relation coverage measure different properties.

Panel (c) separates the outcomes of these initially admitted records for a 32-chunk gap.
The categories are commitment before bridge arrival, commitment at the bridge step or the following step, later commitment, and removal without commitment.
They share one denominator and sum to 100\% for each method.
The share committed at the bridge step or the following step is 81.3\% for \our and 68.2\% for Outcome-only RL.
The dropped category includes removals both before and after the bridge; it therefore does not equal one minus pre-bridge source survival.
These diagnostics use initially admitted supporting-source records as their denominator, which differs from the denominator of the aggregate forced-drop rates in Table~\ref{tab:pending-pressure}.
\endgroup

\begin{figure}[t]
\centering
\includegraphics[width=\linewidth]{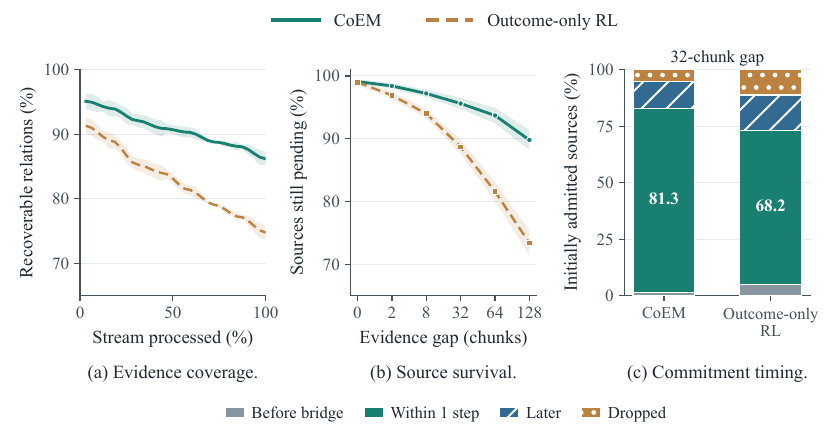}
\caption{Supporting-evidence diagnostics for Qwen3.5-4B on HotpotQA. (a) Recoverable supporting relations at 32 normalized reading positions, before final resolution. (b) Pending-source survival immediately before bridge arrival under the existing gap intervention; labeled gap settings are equally spaced. (c) Commitment timing or removal for initially admitted records containing supporting evidence at gap 32. Within 1 step denotes commitment at the bridge step or the following step. Shaded bands in (a) and (b) show one sample standard deviation across policy seeds; panel (c) reports pooled record proportions.}
\label{fig:evidence-retention}
\end{figure}

\subsection{Reward ablations with matched answer weight}
\label{app:reward-controls}

\begingroup
The Outcome-only RL variant in Table~\ref{tab:ablation} uses $A_i^{\rm ans}$, whereas full \our weights the answer advantage by 0.8.
To isolate the evidence term while keeping the answer-advantage and KL coefficients fixed, we compare three advantages:
\begin{equation}
\begin{aligned}
 A_{i,t}^{\rm no\text{-}evd}
   &=0.8A_i^{\rm ans},\\
 A_{i,t}^{\rm immediate}
   &=0.8A_i^{\rm ans}
     +0.2\left(\frac{r_{i,t}^{\rm evd}}{T}
       -\frac{1}{G}\sum_{j=1}^{G}\frac{r_{j,t}^{\rm evd}}{T}\right),\\
 A_{i,t}^{\rm full}
   &=0.8A_i^{\rm ans}
     +0.2\left(L_{i,t}-\frac{1}{G}\sum_{j=1}^{G}L_{j,t}\right).
\end{aligned}
\label{eq:reward-controls}
\end{equation}
All three retain the same controller, verification gate, generated-token normalization, reference policy, clipping coefficient, and $\beta=10^{-3}$.
The immediate and return-to-go variants use the same $\lambda_v=0.20$, $\lambda_f=0.05$, and $1/T$ normalization.
Only the evidence return differs: the immediate variant centers the current chunk's penalty, while the full variant centers the sum of penalties from that chunk onward.
During final-answer generation the evidence term is zero for every variant.

Table~\ref{tab:reward-controls} reports Qwen3.5-4B results at 800 documents after matched training.
Compared with the control using the same answer weight and no evidence term, \our gains 4.5 F1 points on HotpotQA and 5.6 on 2WikiMultiHopQA.
The return-to-go variant exceeds the immediate evidence reward by 2.4 and 2.9 points, respectively.
These comparisons separate adding evidence feedback from propagating its subsequent penalties to earlier memory decisions.
They evaluate aggregate training effects and do not identify a causal contribution for an individual retained excerpt.
\endgroup

\begin{table}[!htbp]
\centering
\caption{Reward controls with Qwen3.5-4B at 800 documents (mean $\pm$ SD over three seeds). $\alpha$ is the answer-advantage weight; matched controls fix $\alpha=0.8$ and the KL coefficient at $10^{-3}$, and both evidence variants share the same $1/T$ scaling and penalty weights.}
\label{tab:reward-controls}
\vspace{4pt}
{\renewcommand{\arraystretch}{1.15}
\begin{tabular}{lclrr}
\toprule
Training objective & $\alpha$ & Evidence reward & HotpotQA F1 & 2Wiki F1 \\
\midrule
Outcome-only RL              & 1.0 & None          & $80.4\pm0.85$ & $69.9\pm0.85$ \\
Matched $\alpha$, no evidence & 0.8 & None          & $81.0\pm0.70$ & $70.5\pm0.80$ \\
Immediate evidence reward    & 0.8 & Current chunk & $83.1\pm0.60$ & $73.2\pm0.70$ \\
Full CoEM                    & 0.8 & Return-to-go  & $85.5\pm0.46$ & $76.1\pm0.56$ \\
\bottomrule
\end{tabular}}
\end{table}

\FloatBarrier
\section{Computational resources}
\label{app:compute-resources}
Both policy scales train on eight NVIDIA H800 GPUs with full-parameter FSDP and gradient checkpointing; evaluation uses H800 and A100 GPUs.
Appendix~\ref{app:implementation} specifies their shared training and checkpoint-selection settings.
The 1,024-token context-memory budget counts serialized state carried between chunks, separately from execution memory and any external baseline archive.

\begingroup
\paragraph{Training accounting.}
Table~\ref{tab:training-resources} covers full 500-update \our runs, including rollouts, updates, checkpointing, and development evaluation.
GPU-hours count all eight allocated GPUs.
Totals cover seeds 17, 29, and 43; baseline training, ablations, calibration, and exploratory runs are accounted for separately.
\endgroup

\begin{table}[!htbp]
\centering
\caption{Training resources for \our on nodes with eight H800 GPUs.}
\label{tab:training-resources}
\begingroup
\small
\setlength{\tabcolsep}{5pt}
\begin{tabular}{lrrrr}
\toprule
Policy & GPUs/run & Hours/run & GPU-h/run & GPU-h/3 seeds \\
\midrule
Qwen3.5-4B & 8 & 18.0 & 144 & 432 \\
Qwen3.5-9B & 8 & 42.0 & 336 & 1,008 \\
\midrule
Total for the six runs & --- & --- & --- & 1,440 \\
\bottomrule
\end{tabular}
\endgroup
\end{table}

\begingroup
\paragraph{Inference accounting.}
Table~\ref{tab:inference-resources} reports one Qwen3.5-4B inference run on one H800 at batch size one for a 6,400-document, 128-chunk HotpotQA input.
These counts describe this run, not cohort averages.
Input tokens count repeated prompts and current-chunk rereads; output tokens include rejected actions.
Verifier tokens use its own tokenizer.
Serial component times sum to 548\,s, including $360\times4.3$\,ms for verification, aligning with Table~\ref{tab:efficiency}.
Overlapping timers must not be summed; end-to-end wall time includes every operation.
Full-context timing uses its separate tensor-parallel deployment.
\endgroup

\begin{table}[!htbp]
\centering
\caption{Inference resources. Additional calls cover capacity handling and final resolution; verifier memory is included in peak memory.}
\label{tab:inference-resources}
\begingroup
\small
\setlength{\tabcolsep}{6pt}
\begin{tabular}{lr}
\toprule
Resource or operation & Value \\
\midrule
Regular / additional / answer policy calls & 128 / 12 / 1 \\
Total policy calls & 141 \\
Policy input / generated tokens & 1,020,672 / 21,432 \\
Final-answer tokens (included above) & 52 \\
Verifier pairs / processed pair tokens & 360 / 64,800 \\
Policy / verifier / controller time (s) & 540.300 / 1.548 / 6.152 \\
End-to-end wall time (s) & 548.000 \\
Peak allocated device memory (GiB) & 42.60 \\
Verifier-resident device memory (GiB) & 0.55 \\
External historical-source archive (bytes) & 0 \\
\bottomrule
\end{tabular}
\endgroup
\end{table}

\begingroup
\paragraph{Execution record.}
Run records store model/tokenizer revisions, PyTorch/Transformers/verl/vLLM and CUDA versions, GPU model/memory, parallelism settings, data-manifest hash, seed, selected checkpoint, and timing interval.
\endgroup

\clearpage
\definecolor{CoEMCaseRetained}{HTML}{198071}
\definecolor{CoEMCaseLink}{HTML}{326A92}
\definecolor{CoEMCaseLost}{HTML}{B9615E}
\definecolor{CoEMCaseNeutral}{HTML}{8795A0}

\section{Case studies}
\label{app:case-studies}
Two HotpotQA training examples~\citep{yang2018hotpotqa}, excluded from evaluation and tuning, illustrate retention and capacity pressure using the stepwise format of ReMemR1~\citep{shi2025rememr1}.
Source quotations are verbatim; \textcolor{CoEMCaseRetained}{green}, \textcolor{CoEMCaseLink}{blue}, and \textcolor{CoEMCaseLost}{red} mark retained evidence, linking facts, and lost evidence, respectively.
Bold marks new information; the data include both examples' IDs and complete contexts.

\subsection{Delayed commitment across an alias relation}
Figures~\ref{tab:success-demo} and~\ref{fig:success-commitment} follow one source excerpt from admission to a verified answer.
The early biography supplies both a nationality and a spousal relation, but the question uses a different name for the spouse.
The pending set preserves these source sentences until the alias becomes available.

Both cases use Qwen3.5-4B with $B_P=256$, $B_M=768$, $H=8$, $K=1$, $J=2$, and $\eta=0.90$.
The displayed sizes count serialized tokens, including source identifiers, ranks, and separators.
Each state reports committed memory $M$ and the pending set $P$ after the stated operation.
Only question-relevant entries are written out; reported occupancy includes the complete serialized states.

\begin{figure}[!htbp]
\centering
\includegraphics[width=\linewidth]{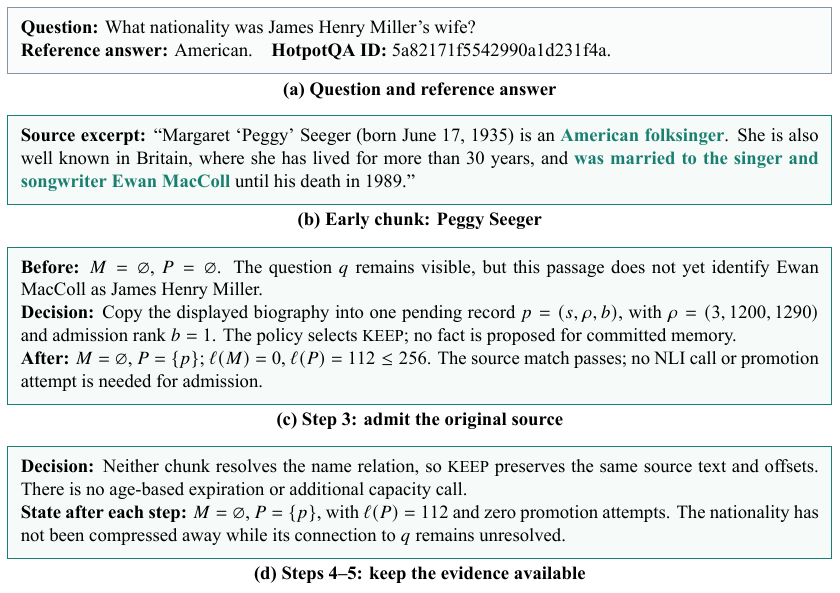}
\caption{Source admission and retention before an alias relation makes the biography relevant to the question.}
\label{tab:success-demo}
\end{figure}

\begingroup
\paragraph{What is preserved.}
The useful unit is the original pair of sentences: one states that Peggy Seeger is American, and the other links her to Ewan MacColl.
Keeping only the spouse's name would retain an entity connection while discarding the attribute requested by the question.
The record therefore carries the text needed both for later relevance assessment and for source verification; its identifier alone would not serve either purpose after the text was removed.

\paragraph{Why commitment waits.}
The biography already supports the nationality and spousal facts, so early commitment is possible in principle.
Here, the policy instead waits for evidence connecting those facts to the question.
This distinction separates support from relevance: source verification checks whether a fact follows from its cited text, while the memory decisions determine whether retaining that fact helps answer $q$.
\endgroup

\clearpage
\begin{figure}[!htbp]
\centering
\includegraphics[width=\linewidth]{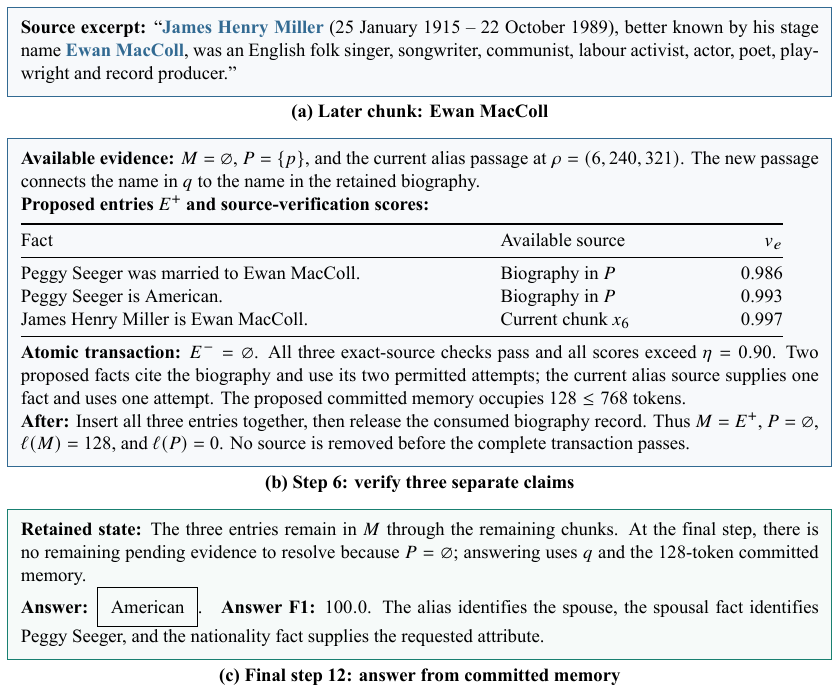}
\caption{Verified commitment after the alias arrives. Each claim is checked against an available source, and the three entries are inserted as one atomic transaction.}
\label{fig:success-commitment}
\end{figure}

\begingroup
\paragraph{Why the transaction is atomic.}
Promoting the two biography claims in separate consuming operations could make the second claim's premise unavailable after the first operation removes $p$.
The joint transaction checks every premise and the resulting memory before consuming that record.
It also respects $J=2$: adding a third biography-supported fact to this transaction would exceed that limit even if every proposed fact were supported.

\paragraph{What source verification establishes.}
The verifier receives the biography for its two factual claims and the current passage for the alias claim; it is not asked to certify the entire multi-hop answer in one compressed statement.
The final answer then follows from the three accepted relations.
A high entailment score is a gate decision, not a guarantee of semantic correctness, so this case illustrates the intended evidence flow rather than replacing the independent verifier audit.

\paragraph{Where premature compression would lose information.}
A memory containing only ``Peggy Seeger was married to Ewan MacColl'' would leave the nationality missing when the alias arrives.
Recalling that same shortened memory cannot restore an attribute absent from it.
Conversely, a method that preserves both relations in a supported summary could answer this example; the case isolates the consequence of losing the attribute, rather than establishing that every immediate summary must fail.
\endgroup

\clearpage
\subsection{Capacity pressure before a linking passage}
Figures~\ref{tab:failure-demo} and~\ref{fig:failure-resolution} follow a biography containing the requested year until capacity handling removes it.
The later film passage supplies the missing entity connection, but the discarded date is no longer available.
The failure depends on the displayed admission and replacement decisions under the fixed budgets.

\begin{figure}[!htbp]
\centering
\includegraphics[width=\linewidth]{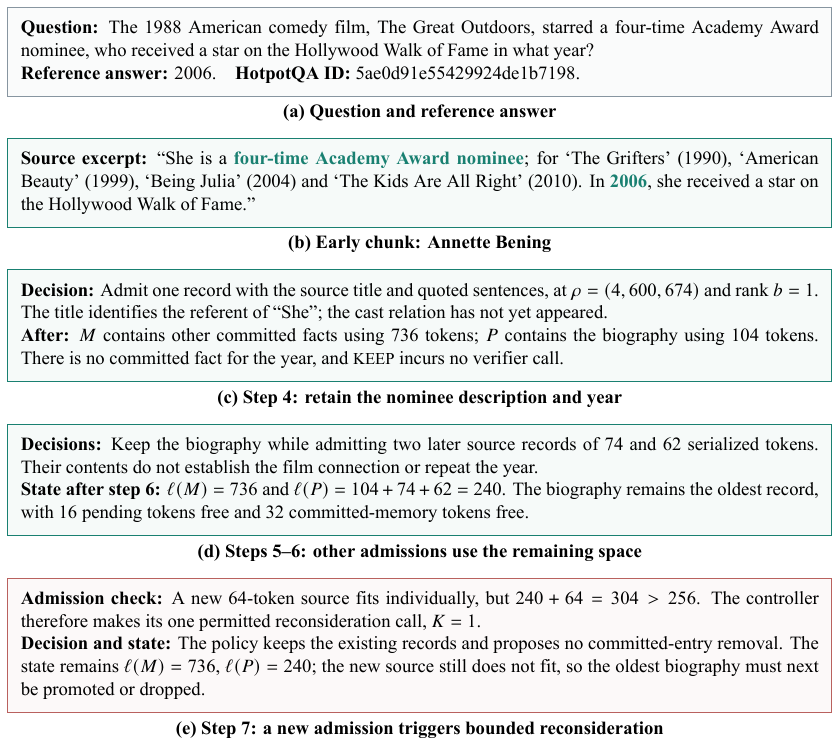}
\caption{Capacity pressure develops before the cast relation arrives. The new source passes its individual size check, but admitting it requires resolution of an existing pending record.}
\label{tab:failure-demo}
\end{figure}

\begingroup
\paragraph{Why this is a capacity event.}
The biography has not expired and has not been judged unsupported.
The trigger is the attempted admission of another eligible source into a nearly full pending set.
The bounded reconsideration pass allows the policy to free space, but its displayed choices leave both stores unchanged, so the controller proceeds to forced resolution.

\paragraph{What the ordering rule decides.}
Admission order selects the biography because it is the oldest remaining record; the rank is not a relevance score.
This rule keeps the controller bounded and ensures progress, but it cannot guarantee that the selected source is less useful than a later source whose relevance is also unresolved.
The critical question is therefore whether the policy can commit its useful content before removal.
\endgroup

\clearpage
\begin{figure}[!htbp]
\centering
\includegraphics[width=\linewidth]{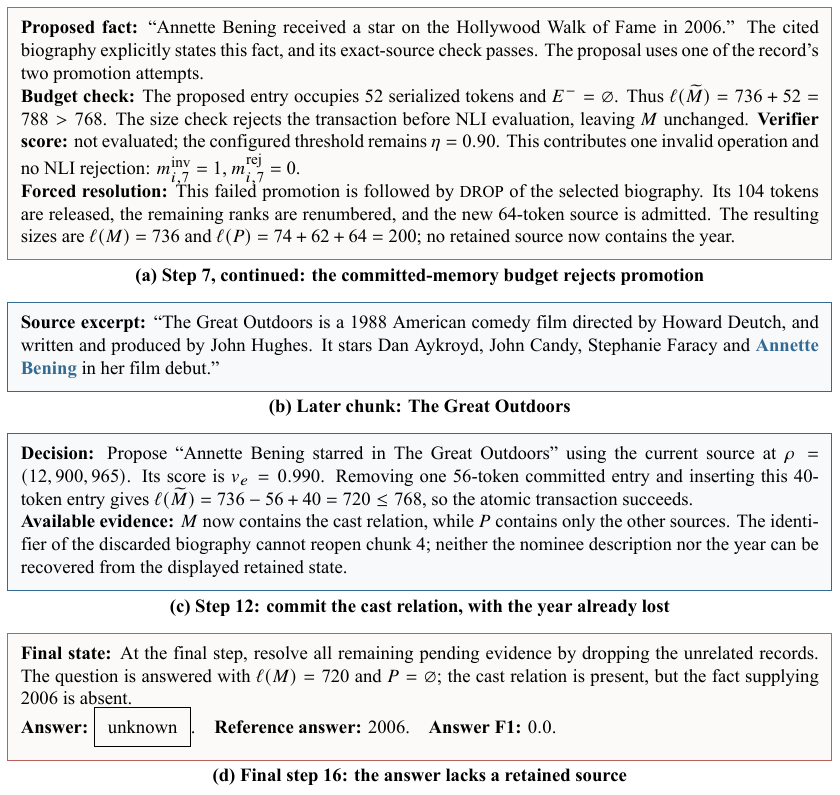}
\caption{A promotion can fail the committed-memory budget check even when its source states the proposed fact. Forced removal then makes the date unavailable when the later cast relation identifies the relevant person.}
\label{fig:failure-resolution}
\end{figure}

\begingroup
\paragraph{Where the failure occurs.}
The decisive loss occurs at step 7: the biography states the date, but the proposed replacement does not fit and the biography is removed under the capacity rule before NLI evaluation.
The later cast passage resolves the entity connection without repeating the date.
The final answer is consequently unsupported by the retained text even though the original stream contained the required evidence.

\paragraph{What could have preserved the year.}
At step 7, deleting suitable committed entries could have made the supported date fit; declining a less useful new admission could instead have preserved the original biography.
The displayed policy takes neither option.
The case therefore shows how bounded capacity interacts with specific retention and replacement decisions, rather than proving that the fixed budget inevitably causes this error or that verification alone can prevent it.

\paragraph{What verification cannot repair.}
Once the source has been dropped, a later proposal citing its old offsets fails the source-availability check before NLI scoring; it should not be assigned a low entailment probability as though the premise were still present.
The verifier cannot recover discarded evidence, and an answer guessed from parametric knowledge would not restore that evidence either.
Aggregate source-survival and capacity-pressure diagnostics are therefore needed to assess how often this failure occurs beyond the individual case.
\endgroup

\clearpage
\section{Formal analysis of memory control under partial observability}
\label{sec:theory}

We analyze the deterministic controller, delayed commitment, and step-level evidence rewards in \our under the fixed context-memory budget.
Execution guarantees require the controller rules specified below; decision and gradient results additionally require the stated probabilistic assumptions.

\subsection{The POMDP and its finite-memory controller}
\label{sec:theory-pomdp}

Let $\mu(q,D,y)$ be the distribution of reading episodes, each consisting of the document stream $D=(x_1,\ldots,x_T)$ and final-answer generation.
We condition on the finite horizon $T$ and obtain mixtures over horizons by averaging the corresponding episodes.
The full state of the environment and controller before reading step $t$ is
\begin{equation}
 s_t=(q,D,y,t,h_{t-1}),\qquad h_{t-1}=(M_{t-1},P_{t-1}).
 \label{eq:full-state}
\end{equation}
The initial distribution draws $(q,D,y)$ from $\mu$ and sets $t=1$, $h_0=(\varnothing,\varnothing)$.
The policy observes the context defined in Equation~\ref{eq:controller}:
\begin{equation}
 \Omega(s_t)=\mathcal C_t=(q,x_t,h_{t-1}).
 \label{eq:observation}
\end{equation}

The step-level action $a_t=(u_{t,1},\ldots,u_{t,d_t})$ records the outputs of all policy calls at reading step $t$, in execution order.
These calls include the initial memory decisions, bounded reconsideration, resolution of the oldest pending evidence under capacity pressure, and resolution of remaining pending evidence at the final step.
The number of calls $d_t$ and the output length of each call have fixed bounds.
Given the context $\xi_{t,r}$ for policy output $u_{t,r}$, the probability of the complete step-level action is
\begin{equation}
 \pi_\theta(a_t\mid\mathcal C_t)
 =\prod_{r=1}^{d_t}\pi_\theta(u_{t,r}\mid\xi_{t,r}).
 \label{eq:macro-policy}
\end{equation}
Each $\xi_{t,r}$ contains the question, current chunk, committed memory, pending set, and bounded controller feedback; preceding outputs affect it through the updated memory and within-step attempt counts, without accumulating a conversation history.
The chunk identifier in $x_t$ supplies the reading index $t$.

The controller computes $h_t=F(h_{t-1},x_t,a_t)$ by applying the memory decisions recorded in $a_t$.
Given $a_t$, the update function $F$ requires no further policy sampling.
With fixed verifier scores and deterministic tie breaking, the state transition kernel is
\begin{equation}
 \mathsf P(s_{t+1}\mid s_t,a_t)
 =\mathbf 1\!\left[s_{t+1}=(q,D,y,t+1,F(h_{t-1},x_t,a_t))\right].
 \label{eq:pomdp-transition}
\end{equation}
Policy sampling determines the selected action, while the transition is deterministic conditional on that action.

After step $T$, the policy generates $\hat y$ from $(q,M_T)$ as in Equation~\ref{eq:answer}, ending the episode.
Each reading step receives reward $(1-\alpha)r_t^{\mathrm{evd}}/T$, where $r_t^{\mathrm{evd}}$ is defined in Equation~\ref{eq:reward}.
The final answer reward is weighted by $\alpha$, giving $\alpha r_{\mathrm{ans}}(\hat y,y)$.

\begin{proposition}[A bounded controller under partial observability]
\label{prop:pomdp}
The process on the full states $s_t$ is Markov.
With a finite token vocabulary and token budget, the visible memory $h_t$ defines a finite-memory controller representation.
This representation need not be a sufficient statistic of the observation history.
\end{proposition}
\begin{proof}
Given $(s_t,a_t)$, the transition and reward are independent of earlier states and actions, establishing the Markov property.
A finite token vocabulary admits finitely many serialized memory strings of length at most $B$, yielding finitely many controller configurations.
Different document prefixes can nevertheless produce the same $h_t$ after discarding a detail that later determines the answer; the conditional distribution of answer-relevant information can differ between those prefixes.
Thus the bounded representation need not be sufficient.
\end{proof}

An exact belief state instead retains the conditional distribution over full states:
\begin{equation}
 \mathfrak b_t(s)=\Pr(s_t=s\mid o_1,a_1,\ldots,o_{t-1},a_{t-1},o_t),
 \label{eq:belief}
\end{equation}
where $o_t=\Omega(s_t)$.
The standard Bayesian recursion updates this distribution using the transition and observation kernels~\citep{kaelbling1998pomdp}.
\our instead uses bounded context memory, with the pending set retaining verbatim source excerpts before commitment.

\paragraph{Commitment as a stopping time.}
Let $\mathcal F_t$ contain the observed contexts and all policy outputs through the end of reading step $t$.
For pending evidence $p$, $t_{\mathrm{arr}}(p)$ is its admission step, whereas $b$ is its current rank in the pending set.
Equation~\ref{eq:stopping-time} gives the time of the first accepted promotion using $p$, or $\infty$ if no such promotion occurs.
For every $t$, the event $\{t_p\leq t\}$ is determined by the accepted operations through $t$ and belongs to $\mathcal F_t$.
Thus $t_p$ is a stopping time; resolving all pending evidence at the final step does not give a finite commitment time to an item that is dropped.

\subsection{Memory budgets, termination, and source verification}
\label{sec:theory-resources}

We hold fixed the question bound $L_q$, chunk bound $C$, context-memory budget $B$, and policy-output bound $W_\pi$.
We also fix the candidate limit $H$, reconsideration limit $K$, per-source promotion-attempt limit $J$, verifier window $W_V$, and final-answer bound $W_{\mathrm{ans}}$.
The input and memory bounds include all serialized source metadata, including chunk identifiers.
The results apply to finite-horizon executions whose source identifiers fit these bounds.

Every committed entry and pending item has positive serialized token length, and both $M$ and $P$ allow a valid empty serialization.
All new or modified memory states are checked using the full serialization, including identifiers, ranks, and separators.
Each proposed fact consumes a promotion attempt for every source record it cites, and rejected proposals count toward the same finite limits.
Once an attempt limit is reached, no further promotion call is permitted for that source: ordinary proposals are rejected, and pending items selected for resolution under capacity pressure or at the final step are dropped.
The termination guarantee requires every such resolution to remove at least one pending item, including after invalid responses; Appendix~\ref{sec:algorithm-details} specifies how these responses are handled.

\begin{theorem}[Bounded state and linear inference cost]
\label{prop:resources}
Under these execution conditions, every policy output sequence satisfies Equation~\ref{eq:budget} and completes every reading step.
A $T$-chunk episode requires $O(T)$ policy and verifier calls.
For fixed model architectures and the preceding fixed limits, context-memory storage is $O(B)$ tokens and inference cost is $O(T)$.
\end{theorem}
\begin{proof}
The empty initial state satisfies the budget.
Every retained state passes the serialized-token checks for $M$, $P$, and their combined state, including after pending removal and rank renumbering.
A rejected committed-memory transaction preserves $M$ by Equation~\ref{eq:transaction}; capacity handling may remove pending items but does not delete committed entries outside a checked transaction.
Induction over the retained states establishes Equation~\ref{eq:budget}.

At most $H$ new source specifications are processed per reading step, including invalid candidates and immediate promotions.
Positive serialized lengths imply at most $B$ pending items at any time.
Each admission permits at most $K$ additional reconsideration calls and at most $B$ resolutions under capacity pressure, since each resolution removes an item.
Resolving all remaining pending evidence at the final step likewise requires at most $B$ such operations.
The shared limit $J$ bounds promotion attempts for every cited source record across all calls in the chunk, and $W_\pi$ bounds the number of proposed facts generated per call.
Consequently, policy and verifier calls per chunk are bounded by constants depending only on the fixed limits.

Summing over $T$ chunks and one answer-generation call proves termination and the call bound.
Bounded policy and verifier input/output lengths imply bounded computation per call for fixed architectures, giving $O(T)$ inference cost and $O(B)$ context-memory storage.
\end{proof}

The linear bound includes all reconsideration and verifier calls, whose constant-factor overhead is included in the end-to-end efficiency measurement.
The current chunk and policy outputs are used within the reading step; only $h_t=(M_t,P_t)$ is carried into the next step.
The controller interface excludes a growing conversation history or source archive.

\begin{proposition}[Source verification and atomic memory updates]
\label{prop:transaction}
Every fact in committed memory is either unchanged from the preceding state or has passed the source-availability, source-matching, and NLI checks in Equation~\ref{eq:gate}.
These checks apply at the most recent insertion or modification of that fact.
A failed update preserves every preexisting committed entry and its source identifiers.
\end{proposition}
\begin{proof}
The initial committed memory is empty.
For a successful update, entries in $M\setminus E^-$ are unchanged and every new entry $e\in E^+$ passes $g(e,S_e)=1$.
These are the only entries retained by Equation~\ref{eq:transaction}.
For a failed update, that equation sets $M'=M$.
Induction gives the two statements for every transition.
\end{proof}

Passing the checks does not guarantee source support; a bound on unsupported committed facts requires an additional assumption about verifier errors.
Let $U_j$ indicate that the $j$th accepted insertion or modification is unsupported by its cited source.
Suppose the audited deployment distribution satisfies
\begin{equation}
 \Pr(U_j=1\mid\mathcal F_{j-1},\text{the }j\text{th change is accepted})\leq\varepsilon_V,
 \label{eq:conditional-calibration}
\end{equation}
where $\mathcal F_{j-1}$ includes all preceding proposals and gate outcomes.
This condition uniformly bounds errors among accepted insertions and modifications; a high NLI score alone does not establish the bound.

\begin{corollary}[Accepted-error control]
\label{cor:accepted-error}
Under Equation~\ref{eq:conditional-calibration}, among at most $n$ accepted fact insertions or modifications,
\begin{equation}
 \mathbb E\!\left[\sum_{j=1}^{n}U_j\right]\leq n\varepsilon_V,
 \qquad
 \Pr\!\left(\sum_{j=1}^{n}U_j\geq1\right)\leq\min\{1,n\varepsilon_V\}.
 \label{eq:accepted-error}
\end{equation}
\end{corollary}
\begin{proof}
Set $U_j=0$ when the episode contains fewer than $j$ accepted fact insertions or modifications.
Conditioning on the previous history and whether the change exists gives $\mathbb E[U_j]\leq\varepsilon_V$.
Linearity of expectation proves the first bound.
The union bound, or the Markov inequality applied to the sum, gives the second.
No independence among accepted errors is required.
\end{proof}

An aggregate audit measures errors among accepted entries but does not establish the uniform conditional assumption in Equation~\ref{eq:conditional-calibration}; neither does freezing or calibrating the classifier.
Faithfulness means that the supplied passage supports the statement, while answer rewards measure its usefulness for the task.

\subsection{The value of retaining unresolved evidence}
\label{sec:theory-delay}

Delayed commitment can help when later context changes which retention decision is useful.
We examine this benefit for a source excerpt that remains in the pending set until later context arrives.
Let $X$ denote currently available information, $Z$ the later observation, and $Y$ the hidden variable determining which final retention decision is useful.
Let $\mathcal A_0$ be a finite nonempty set of feasible final actions, and let $\mathcal L(a,Y)\in[0,1]$ be a bounded task loss.
Both the immediate and later decisions use the same action set and final loss, and the later decision can ignore $Z$.
The optimal conditional risks before and after that observation are
\begin{align}
 R_0(X)&=\min_{a\in\mathcal A_0}\mathbb E[\mathcal L(a,Y)\mid X],\\
 R_1(X)&=\mathbb E\!\left[\min_{a\in\mathcal A_0}\mathbb E[\mathcal L(a,Y)\mid X,Z]\,\middle|\,X\right].
 \label{eq:decision-risks}
\end{align}
The expectations exist by boundedness, and finite action sets admit measurable minimizers with fixed tie breaking.
All conditional statements below hold almost surely.

\begin{theorem}[Conditional value of delayed commitment]
\label{prop:voi}
Under the preceding assumptions, $\Delta(X)=R_0(X)-R_1(X)\geq0$.
Let $\kappa(X)\geq0$ be the additive conditional expected cost of waiting, measured in the same loss units.
This cost includes pending-set occupancy, displaced evidence, and extra computation.
If this cost is independent of the final action, the later decision has no greater total conditional risk exactly when
\begin{equation}
 \Delta(X)\geq\kappa(X).
 \label{eq:waiting-threshold}
\end{equation}
\end{theorem}
\begin{proof}
Choose $a_0(X)$ attaining the minimum in $R_0(X)$.
For every realized later observation, minimizing over $\mathcal A_0$ gives
\[
 \min_{a\in\mathcal A_0}\mathbb E[\mathcal L(a,Y)\mid X,Z]
 \leq\mathbb E[\mathcal L(a_0(X),Y)\mid X,Z].
\]
Taking conditional expectation given $X$ and using the tower property yields $R_1(X)\leq R_0(X)$.
The later total risk is $R_1(X)+\kappa(X)$, so comparing it with $R_0(X)$ proves Equation~\ref{eq:waiting-threshold}.
\end{proof}

\paragraph{An explicit retention example.}
Let $Y\in\{0,1\}$ indicate whether a source detail will be needed, with $p=\Pr(Y=1\mid X)$.
Keeping an unneeded detail incurs cost $c_{\mathrm{ret}}\in[0,1]$, and discarding a needed detail incurs cost $c_{\mathrm{miss}}\in[0,1]$.
The other two outcomes have zero loss, so
\begin{equation}
 R_0(X)=\min\{(1-p)c_{\mathrm{ret}},\;p c_{\mathrm{miss}}\}.
 \label{eq:binary-retention}
\end{equation}
If the later observation identifies $Y$, the later risk $R_1(X)$ is zero.
Waiting is beneficial when its full cost is smaller than the risk in Equation~\ref{eq:binary-retention}.
This is a local decision comparison with the same feasible final actions; it does not establish the superiority of a complete learned policy whose allocation between $M$ and $P$ changes other retention decisions.

\paragraph{Retention under the capacity bound.}
The preceding comparison assumes that the source remains available while the controller waits.
Pending evidence has no age limit.
To state a sufficient condition for retention, let $\bar\ell$ be a nonnegative additive upper bound on the length of the pending serialization.
It includes separators, the source text and identifiers of each pending item, and the maximum token cost of its rank among $1,\ldots,B_P$.
This bound remains monotone under removal, even when renumbering changes tokenization lengths.
Fix the pending set $P_{t_0}$ after a completed reading step, with $t_0<u<T$.
Let $\mathcal U_{t_0:v}$ contain every distinct source excerpt selected for admission during steps $t_0+1,\ldots,v$, including all admissions within each step.
If
\begin{equation}
 \bar\ell(P_{t_0}\cup\mathcal U_{t_0:v})\leq B_P\quad\text{for every }t_0<v\leq u,
 \label{eq:survival-condition}
\end{equation}
and the policy applies \textsc{Keep} to $p\in P_{t_0}$ without releasing it after another promotion, then no capacity rule removes $p$ through the end of step $u$.
By monotonicity, every intermediate admission fits even without space released by removals, so no forced capacity action is needed.
The restriction $u<T$ excludes resolution of all remaining pending evidence at the final step.

Equation~\ref{eq:survival-condition} supplies a sufficient condition for the source excerpt to remain available, as assumed in Theorem~\ref{prop:voi}.
It also motivates measuring retention across evidence gaps and pending-set allocations.

\subsection{Step-level evidence rewards and final answer rewards}
\label{sec:theory-learning}

LongRLVR analyzes why final-answer rewards provide a weak learning signal for intermediate evidence decisions~\citep{chen2026longrlvr}.
We study the same question for source-supported facts in committed memory.
The source-selection reward in LongRLVR supervises a different quantity from the step-level evidence reward in \our.
Our analysis concerns whether cited sources support a proposed fact; final answer rewards determine which supported facts are useful.

Consider $m$ required memory decisions.
At decision $j$, the policy proposes a source-supported fact when $Z_j=1$ and an unsupported fact when $Z_j=0$.
For this analysis, assume independent Bernoulli choices $Z_j\sim\mathrm{Bernoulli}(p_j)$ with $p_j=\sigma(\vartheta_j)$ and separate logits $\vartheta_j$.
Let $\Gamma_j\in\{0,1\}$ indicate gate acceptance, with fixed rates
\begin{equation}
 a_j=\Pr(\Gamma_j=1\mid Z_j=1),\qquad
 b_j=\Pr(\Gamma_j=1\mid Z_j=0),\qquad a_j>b_j.
 \label{eq:verifier-separation}
\end{equation}
Gate outcomes are conditionally independent across decisions given the $Z_j$.

An independent variable $U\sim\mathrm{Bernoulli}(c)$ represents the remaining answer-generation success.
The simplified final answer reward is $R^{\mathrm{ans}}=U\prod_{k=1}^{m}Z_k\Gamma_k$, so success requires every fact in the chain to be source-supported and accepted.
Each verifier rejection incurs the evidence penalty $r_j=-\lambda_v(1-\Gamma_j)$.
The model holds the relevant proposals and verifier discrimination rates fixed.
It therefore analyzes the choice between supported and unsupported facts without modeling how the policy discovers useful facts.

\begin{theorem}[A source-support learning signal without final-answer success]
\label{prop:local-signal}
Under these assumptions,
\begin{align}
 \frac{\partial\mathbb E[R^{\mathrm{ans}}]}{\partial\vartheta_j}
 &=p_j(1-p_j)\,c a_j\prod_{k\ne j}p_k a_k,
 \label{eq:sparse-gradient}\\
 \frac{\partial\mathbb E[r_j]}{\partial\vartheta_j}
 &=\lambda_v(a_j-b_j)p_j(1-p_j).
 \label{eq:local-gradient}
\end{align}
Thus the expected combined return $\alpha R^{\mathrm{ans}}+(1-\alpha)T^{-1}\sum_k r_k$ has derivative
\begin{equation}
 p_j(1-p_j)\left[\alpha c a_j\prod_{k\ne j}p_k a_k
 +\frac{(1-\alpha)\lambda_v}{T}(a_j-b_j)\right].
 \label{eq:combined-gradient}
\end{equation}
For $p_j\in(0,1)$, $\lambda_v>0$, and $\alpha<1$, the second term is positive.
This term contains no factor for the probability that all other required facts are supported and accepted.
\end{theorem}
\begin{proof}
Conditional independence gives $\mathbb E[R^{\mathrm{ans}}]=c\prod_k p_k a_k$.
Differentiating with $\partial p_j/\partial\vartheta_j=p_j(1-p_j)$ proves Equation~\ref{eq:sparse-gradient}.
The acceptance probability at decision $j$ is $p_j a_j+(1-p_j)b_j$, so
\[
 \mathbb E[r_j]=-\lambda_v\{1-p_j a_j-(1-p_j)b_j\}.
\]
Differentiation gives Equation~\ref{eq:local-gradient}.
The other evidence penalties do not depend on $\vartheta_j$ by the separate-logit and independence assumptions.
Linearity proves Equation~\ref{eq:combined-gradient}.
\end{proof}

Verification can provide useful feedback on an unsupported fact even when missing evidence at another hop causes the final answer to be incorrect.
The model compares alternatives at a fixed decision, consistent with the absence of a per-write bonus in Equation~\ref{eq:reward}.
The positive evidence term requires verifier discrimination $a_j>b_j$ and $0<p_j<1$.
The independent faithfulness audit measures aggregate discrimination, whereas the condition $a_j>b_j$ must hold at each decision for the theorem.
The audit does not establish that all answer-relevant evidence has been retained.

\begin{proposition}[Effect of group centering]
\label{prop:group-centering}
Consider $G\geq2$ independent rollouts with the same input, and let $R_i$ denote the return of each rollout.
Let $S_i=\nabla_\theta\log\pi_\theta(\tau_i)$ be the corresponding score, with $\mathbb E[S_i]=0$.
With $\bar R=G^{-1}\sum_iR_i$,
\begin{equation}
 \mathbb E\!\left[\frac1G\sum_{i=1}^{G}(R_i-\bar R)S_i\right]
 =\frac{G-1}{G}\,\nabla_\theta\mathbb E[R].
 \label{eq:group-centering-proof}
\end{equation}
\end{proposition}
\begin{proof}
For $k\ne i$, independence implies $\mathbb E[R_kS_i]=\mathbb E[R_k]\mathbb E[S_i]=0$.
The self term contributes $G^{-1}\mathbb E[R_iS_i]$ to $\mathbb E[\bar R S_i]$.
Subtracting from $\mathbb E[R_iS_i]$ gives $(1-G^{-1})\mathbb E[R_iS_i]$.
Averaging over $i$ and applying the score-function identity proves the result.
\end{proof}

The same argument applies to return-to-go at a fixed decision index.
Conditional on the observed prefix, earlier rewards contribute zero expected score-function gradient.
Group centering therefore preserves the direction of the evidence-learning signal at the rollout policy for an unclipped score-function objective without KL regularization.
The practical objective in Equation~\ref{eq:clipped_objective} also uses token averaging, clipping, and KL regularization.
Equation~\ref{eq:group-centering-proof} does not establish the same gradient direction or convergence for these modified updates.
The analysis explains how verification supplies intermediate supervision; answer accuracy and the quality of commitment timing remain empirical questions.

\clearpage
\section{Controller and optimization details}
\label{sec:algorithm-details}

\subsection{Notation and action representation}

Within reading step $t$, $M$ and $P$ denote the current committed memory and pending set; $M_t$ and $P_t$ denote the state after the step is completed.
A new candidate specifies source offsets $(t,l,r)$ and an action in $\{\textsc{Promote},\textsc{Keep},\textsc{Drop}\}$; promotion additionally supplies a proposed fact.
Offsets $[l,r)$ refer to token positions in the current chunk under the backbone tokenizer.
The controller copies the exact sequence $x_t[l:r]$ and reconstructs its displayed text, rather than accepting a quotation generated by the policy.

The admission rank $b$ records the position of an excerpt in the pending set, ordered by admission.
New excerpts are appended in admission order, and remaining ranks are renumbered from 1 after removal without changing that order.
Admission order is therefore recoverable from $P$ itself and requires no uncounted global counter.
The rank determines which pending item is resolved first under capacity pressure; it does not measure confidence or impose an age limit.
Its representation counts toward $\ell(P)$.
A fact may cite several spans, with each retained excerpt and identifier stored and counted separately.
The limit $H$ counts all new source specifications per reading step, including invalid candidates and immediate promotions.

Source text is available only from the current chunk $x_t$ and verbatim excerpts retained in the pending set $P$.
An earlier reference $(t',l',r')$, $t'<t$, is usable only if the complete cited span remains in a pending excerpt.
After that excerpt is removed, its identifier in $M$ cannot recover the source text.
The identifiers $\rho$ therefore record provenance and select an \emph{available} premise for $V_\phi$; they do not retrieve historical text.

The action format requires a separate proposed fact for each independently checkable relation.
A conjunction such as \emph{A was born in B and founded C} should occupy two fields, each with its supporting spans.
The parser checks the action format but cannot always determine the number of independently checkable relations in a sentence.
The verifier therefore evaluates the full proposed fact.
For a fact requiring several excerpts, the premise $S_e$ concatenates them with source separators.
The complete pair $(S_e,f)$ is then retokenized for the verifier, and premises are never truncated to fit $W_V$.

\subsection{Memory decisions and atomic memory updates}

\paragraph{Candidate validation.}
The controller validates at most $H$ new source specifications per reading step and copies the corresponding original tokens in the order specified by the policy.
Invalid candidates still count toward the limit and receive the invalid-operation penalty, but cannot displace existing evidence.
Validation changes neither $M$ nor $P$.
Candidates selected for retention are submitted for pending admission; immediate promotions use the same validated source text.

\paragraph{Memory decisions and atomic transactions.}
Immediate promotions and actions on existing pending evidence follow the order specified by the policy.
\textsc{Keep} preserves a pending excerpt, whereas \textsc{Drop} removes the entire item.
Promotion constructs $(E^-,E^+)$ and the premises $S_e$, then applies Equations~\ref{eq:gate} and~\ref{eq:transaction}.
A successful transaction updates $M$ before releasing the pending excerpts selected for removal.
For the pending evidence $p$ selected for promotion, a successful transaction releases $p$; additional supporting excerpts may be retained or released together.
A failed transaction preserves $M$; its supporting pending excerpts also remain unchanged unless capacity pressure or resolution at the final step requires their removal.
Facts rejected by the verifier contribute to $m_{i,t}^{\mathrm{rej}}$ only after passing all non-entailment checks.
An invalid operation contributes to $m_{i,t}^{\mathrm{inv}}$ and does not also count as an NLI rejection.

For counting attempts, a source record is a validated source excerpt identified by its chunk and token offsets.
Its promotion-attempt limit $J$ is shared across ordinary calls, capacity handling, and resolution at the final step within one reading step; rephrasing a fact does not reset it.
Each proposed fact consumes one attempt for every source record it cites, so a transaction with multiple facts can consume several attempts.
A proposal exceeding the remaining limit is rejected before verification.
Transactions using multiple sources check every premise before releasing any supporting excerpt.
If an accepted action releases a source required by a later action, the later action fails the source-availability check.
Immediate promotions also count toward $H$ and $J$.

\paragraph{Capacity pressure and the final step.}
A validated candidate that fits within Equation~\ref{eq:budget} requires no additional policy call.
A candidate too large to fit alone is rejected before any existing source is removed.
Otherwise, each capacity event permits at most $K$ additional calls to reconsider pending evidence using $q,x_t,M,P$, under the same transaction rules and attempt limits.
If the candidate still does not fit, the controller repeatedly selects the oldest pending item for resolution.
When a promotion attempt remains, the policy can promote or drop that item, with \textsc{Keep} disabled.
A failed promotion or exhausted attempt limit causes the selected item to be dropped, so each such resolution removes at least one pending item.
The candidate is admitted only after the complete serialization of committed memory and the pending set passes the budget checks.

After processing the final chunk, the controller resolves all remaining pending evidence in admission order with \textsc{Keep} disabled.
Each item is promoted under the same verification and transaction rules or dropped.
Resolution at the final step uses the remaining promotion attempts without resetting $J$.
Promotions using multiple sources check all still-available premises before releasing the selected pending excerpts together.
During resolution under capacity pressure or at the final step, an invalid response, including a malformed action or a disallowed \textsc{Keep}, causes the selected pending item to be dropped without another policy call.
Under these rules, the step ends with $P_T=\varnothing$, and the policy answers from $q$ and $M_T$.

\paragraph{Capacity enforcement and bounded calls.}
Capacity checks use the complete serialization after each update, including source identifiers, separators, and renumbered pending ranks.
Committed-entry deletion and insertion are checked as one transaction; rejection leaves the previous $M$ unchanged.
Further attempts remain subject to the existing $H,K,J$ limits; if a pending item must be resolved but has no promotion attempt left, it is dropped without another policy call.
Question, chunk, policy-output, and final-answer lengths are bounded by $L_q,C,W_\pi,W_{\mathrm{ans}}$, respectively.
Each additional call uses the current bounded observation and feedback without accumulating transcripts from earlier calls.
Appendix~\ref{app:implementation} provides the numerical settings and decoding parameters.

\subsection{Token-level objective and credit assignment}

Let $\mathcal I_i$ index policy-generated tokens in trajectory $i$, including memory decisions, proposed facts, entry-removal choices, and the answer.
Their total count is $N_{\mathrm{gen}}=\sum_i|\mathcal I_i|$.
For $k\in\mathcal I_i$, $\xi_{i,k}$ contains the current bounded observation, controller feedback, and tokens generated earlier within the same policy call.
Source text and verifier feedback are fixed inputs during differentiation.
All bounded calls within a chunk share the same index $t$ in $A_{i,t}$.
The advantage therefore compares complete memory transitions without assigning a separate causal score to each source.

For vocabulary $\mathcal V$, the divergence in Equation~\ref{eq:clipped_objective} is
\begin{equation}
 D_{\mathrm{KL}}\!\left(\pi_\theta(\cdot\mid\xi)\,\|\,\pi_{\mathrm{ref}}(\cdot\mid\xi)\right)
 =
 \sum_{w\in\mathcal V}\pi_\theta(w\mid\xi)
 \log\frac{\pi_\theta(w\mid\xi)}{\pi_{\mathrm{ref}}(w\mid\xi)}.
 \label{eq:kl-detail}
\end{equation}
Equation~\ref{eq:clipped_objective} averages this divergence over sampled token contexts.
The reference policy $\pi_{\mathrm{ref}}$ and rollout policy $\pi_{\theta_{\mathrm{old}}}$ remain fixed within each optimization batch; the numerator of $r^\pi_{i,k}$ and the KL term use the current policy $\pi_\theta$.
Completed-rollout advantages receive no gradient.

Consider an excerpt admitted at chunk 2, kept at chunk 3, and used in an unsupported promotion at chunk 5.
The resulting penalty appears in every $L_{i,t}$ with $t\leq5$, including the admission and retention steps, but not in later returns.
It therefore assigns statistical credit to earlier memory decisions as well as the eventual proposal; the answer reward evaluates their task utility.
Supported proposals receive no bonus per write, so storing additional supported but irrelevant facts cannot by itself increase the evidence reward.

When all trajectories receive the same answer score, the centered answer component is zero, but different rejection histories can yield nonzero evidence advantages.
If both components are constant, the combined advantage is zero; KL regularization can still contribute to the update.
Centering does not divide by the standard deviation, preserving the chosen scales $\lambda_v,\lambda_f,\alpha$.
During answer generation, the evidence return is zero because no memory update follows.

\end{document}